\documentclass[11pt]{article}
\makeatletter
\def\input@path{{styles/}}
\makeatother
\usepackage{arxiv-preprint}
\usepackage[T1]{fontenc}
\usepackage{times}

\usepackage[utf8]{inputenc}
\usepackage{microtype}
\usepackage{url}
\usepackage{xcolor}
\usepackage[hidelinks]{hyperref}

\usepackage{amsmath,amssymb,amsthm,mathtools}
\usepackage{caption}
\usepackage{graphicx}
\usepackage{booktabs,tabularx,array}
\usepackage{enumitem}
\usepackage{algorithm}
\usepackage{algorithmic}
\usepackage{float}
\usepackage{placeins}

\DeclareMathOperator{\sign}{sign}
\newcommand{\sgn}{\operatorname{sign}}
\newcommand{\polar}{\operatorname{polar}}
\newcommand{\op}{\mathrm{op}}
\newcommand{\F}{\mathrm{F}}

\newcommand{\R}{\mathbb{R}}
\newcommand{\E}{\mathbb{E}}

\newcommand{\tr}{\operatorname{tr}}
\newcommand{\diag}{\operatorname{diag}}

\usepackage{mdframed}
\mdfdefinestyle{thmstyle}{%
  linecolor=black!20,
  backgroundcolor=black!5,
  linewidth=0.5pt,
  roundcorner=3pt,
  skipabove=6pt,
  skipbelow=6pt,
  innertopmargin=5pt,
  innerbottommargin=5pt,
  innerleftmargin=7pt,
  innerrightmargin=7pt,
}
\newmdtheoremenv[style=thmstyle]{theorem}{Theorem}
\newmdtheoremenv[style=thmstyle]{lemma}[theorem]{Lemma}
\newmdtheoremenv[style=thmstyle]{corollary}[theorem]{Corollary}
\newmdtheoremenv[style=thmstyle]{proposition}[theorem]{Proposition}
\newmdtheoremenv[style=thmstyle]{assumption}[theorem]{Assumption}
\theoremstyle{definition}
\newmdtheoremenv[style=thmstyle]{definition}[theorem]{Definition}
\newmdtheoremenv[style=thmstyle]{example}[theorem]{Example}
\theoremstyle{remark}
\newmdtheoremenv[style=thmstyle]{remark}[theorem]{Remark}

\hypersetup{pdftitle={SignMuon: Communication-Efficient Distributed Muon Optimization for LLM Pretraining}}

\title{SignMuon: Communication-Efficient\\Distributed Muon Optimization for LLM Pretraining}
\author{%
  \normalsize Neel Mishra\textsuperscript{1}\thanks{Email: \href{mailto:MishraNeel99@gmail.com}{\nolinkurl{MishraNeel99@gmail.com}}},
  Jichao Jiang\textsuperscript{2},
  Kushagara Trivedi\textsuperscript{3}\\[0.3em]
  \normalsize HanQin Cai\textsuperscript{2},
  Aritra Dutta\textsuperscript{2}\thanks{Email: \href{mailto:aritra.dutta@ucf.edu}{\nolinkurl{aritra.dutta@ucf.edu}}},
  Pawan Kumar\textsuperscript{3}\thanks{Email: \href{mailto:pawan.kumar@iiit.ac.in}{\nolinkurl{pawan.kumar@iiit.ac.in}}}\\[0.7em]
  \small \textsuperscript{1}Microsoft, India\\
  \small \textsuperscript{2}University of Central Florida, USA\\
  \small \textsuperscript{3}IIIT, Hyderabad, India
}
\hypersetup{pdfauthor={Neel Mishra, Jichao Jiang, Kushagara Trivedi, HanQin Cai, Aritra Dutta, Pawan Kumar}}
\date{}

\begin{document}
\maketitle

\begin{abstract}
Full-precision gradient exchange can limit data-parallel training. We propose SignMuon, a distributed optimizer that combines Muon's matrix update directions with signSGD's majority vote. Each worker approximates the polar factor of its local momentum matrix and sends the entrywise signs. Workers aggregate them through int8 AllReduce or packed 1-bit AllGather and apply the same update. We prove conditional stationarity bounds for the resulting vote and separate polar-solver error from the loss introduced by signing. The cubic solver bound allows arbitrarily small positive singular values; the optimizer bounds require explicit vote-alignment assumptions. We evaluate three language-model architectures with up to 1.56B parameters on eight GPUs. SignMuon attains the lowest validation log-perplexity among the tested signed optimizers at the last shared token count and elapsed time in each pretraining comparison. These results concern the recorded learning rates and training budgets. Our communication analysis accounts for both the size of each packed buffer and its replication across workers.
\end{abstract}

\section{Introduction}
\label{sec:intro}

Synchronous data-parallel training alternates local gradient computation with communication across workers. As models grow, exchanging full-precision gradients or updates can consume a substantial part of each iteration. An optimizer therefore affects training time through both the updates it computes and the information it exchanges. Reducing message size helps only if the resulting updates still make useful progress.

SignSGD addresses communication by replacing each gradient coordinate with its sign and combining workers' signs through majority vote \citep{bernstein2018signsgd}. A worker can encode each sign in one bit. Under suitable noise assumptions, the majority-vote bound reduces the stochastic term by $1/\sqrt M$, where $M$ is the worker count. The update, however, treats each coordinate separately. It does not directly use the matrix structure of weight tensors.

Muon takes a different approach to the update direction. It forms a momentum matrix and approximates its polar factor through Newton--Schulz (NS) iterations \citep{jordan6muon,shen2025muon}. This operation uses the matrix structure of the parameters and controls the singular values of the direction. Full-precision distributed implementations still exchange floating-point information, so the local update rule alone does not remove the communication bottleneck.

We combine these operations in SignMuon. Each worker computes a polar direction from its local momentum, takes its entrywise signs, and participates in a majority vote. All workers apply the resulting sign matrix. The polar computation requires no messages: workers perform it before exchanging signs. We analyze the signed vote, with normalization for the stationarity bounds.

The representation and the collective both determine communication cost. An int8 AllReduce sends one byte per sign entry and sums the signs within the collective. A packed AllGather sends one bit per entry but gives every worker all of the other workers' buffers, after which each worker computes the vote locally. The packed buffer is eight times smaller, yet AllGather replicates it across workers. We analyze these costs separately instead of equating a one-bit representation with a fixed reduction in network traffic.

Our analysis concerns the implemented momentum--polar--sign pipeline. Normalizing the vote bounds its spectral norm, while an explicit alignment condition controls progress. We also give a conditional fixed-batch rate by tracking momentum error. The local polar solver admits an inner-product error bound without a lower bound on positive singular values. A separate signing-loss term connects polar approximation to the explicit alignment assumptions used in our stationarity guarantees.

Across three architectures and three pretraining corpora on eight GPUs, SignMuon reaches the lowest mean endpoint log-perplexity among the tested sign-based optimizers within the recorded 30-minute training window. Token-efficiency and wall-clock curves describe progress under the reported configurations.

We describe the update and analyze its stationarity under explicit sign-error assumptions, then compare pretraining loss against tokens and elapsed time. The appendices give implementation details, experimental protocols, and proofs.

\section{Related work}
\label{sec:related_work}

\paragraph{Sign methods and distributed Lion.}
\citet{bernstein2018signsgd} analyze signSGD with majority vote. Lion uses signed momentum updates \citep{chen2023symbolic}; Distributed Lion communicates binary or low-precision worker updates \citep{liu2024distributedlion}. It is a close comparator because it also combines local momentum with compressed aggregation. \citet{jiang2025lion} establish centralized and distributed Lion rates under specified parameter schedules and analyze two-way compression, including unbiased randomized signs. Their guarantees depend on the update and compression rule and do not transfer to deterministic polar signs. \citet{sun2023momentum} establish small-batch convergence for momentum signSGD under weaker noise assumptions. SignMuon differs through the local polar transform, whose sign alignment needs a separate argument. Our pretraining experiments include SignLion with packed sign aggregation. OLion \cite{wang2026olionapproachinghadamardideal} independently applies an entrywise sign after orthogonalization; however, it does not develop or evaluate communication-efficient distributed aggregation.

QSGD and TernGrad quantize gradients \citep{alistarh2017qsgd,wen2017terngrad}; Deep Gradient Compression sparsifies them \citep{lin2017deep}; PowerSGD exchanges low-rank factors \citep{vogels2019powersgd}; and error feedback compensates for compression error \citep{stich2018sparsified}. SignMuon instead signs locally computed polar directions before majority voting.

\paragraph{Matrix updates and momentum.}
$K$-FAC and Shampoo construct matrix preconditioners \citep{martens2015optimizing,gupta2018shampoopreconditionedstochastictensor}. Muon uses momentum and polar-factor iterations \citep{jordan6muon,shen2025muon,article,bowie1971}. SignMuon retains this local transformation and compresses its output. Momentum \citep{POLYAK19641,pmlr-v28-sutskever13} and weight decay \citep{loshchilov2017decoupled} remain distinct implementation choices.

\paragraph{Recent analyses of polar signs and inexact solvers.}
\citet{qian2026degenerate} analyze inexact linear minimization oracle (LMO) updates when positive singular values approach zero. Appendix~\ref{app:related-theory} relates this work to our SignMuon analysis and its alignment assumptions.

\section{SignMuon}
\label{sec:algorithms}
\label{sec:method_overview}

Let $W_t\in\mathbb{R}^{m\times n}$ denote a parameter block at iteration $t$, with $d=mn$ entries. Worker $p\in\{1,\ldots,M\}$ computes a stochastic gradient $G_t^{(p)}$ from its local mini-batch. We use $\|\cdot\|_{\op}$, $\|\cdot\|_{\F}$, and $\|\cdot\|_1$ for the spectral, Frobenius, and entrywise $\ell_1$ norms; $\langle A,B\rangle=\tr(A^\top B)$. For a compact singular value decomposition $X=U\Sigma V^\top$, we define $\polar(X)=UV^\top$ on the nonzero singular subspace and $\polar(0)=0$.

\subsection{Local direction and global vote}

Worker $p$ first updates its momentum buffer:
\begin{equation}
\widetilde G_t^{(p)}=G_t^{(p)}+\lambda W_t,
\qquad
M_{t+1}^{(p)}=\beta M_t^{(p)}+(1-\beta)\widetilde G_t^{(p)}.
\end{equation}
Here $\beta$ is the momentum coefficient and $\lambda$ is the weight-decay coefficient. The worker approximates the momentum's polar factor and extracts its signs:
\begin{equation}
U_t^{(p)}=\mathrm{PolarNS}(M_{t+1}^{(p)};K,\varepsilon,\textsc{scale},P),
\qquad S_t^{(p)}=\sign(U_t^{(p)}).
\end{equation}
The routine uses $K$ Newton--Schulz iterations and either Frobenius scaling or a spectral-norm estimate from $P$ power iterations. Algorithm~\ref{alg:polar-ns} normalizes the input, applies the cubic matrix iteration, and returns the approximate polar factor. For spectral scaling, Algorithm~\ref{alg:power-iter} alternates multiplication by the matrix and its transpose, normalizing each vector before returning a norm estimate. A finite power-iteration estimate need not upper-bound the spectral norm; the solver guarantee uses a valid upper bound, as provided by Frobenius scaling. We take signs entrywise and map zero to $+1$ in the binary communication algorithms, including ties in the majority vote.

Workers then compute
\begin{equation}
\label{eq:mv_def}
\bar S_t=\sign\!\left(\sum_{p=1}^{M}S_t^{(p)}\right),
\qquad W_{t+1}=W_t-\eta_t\bar S_t.
\end{equation}
For a single worker, $\bar S_t=S_t^{(1)}$. Algorithm~\ref{alg:sign-muon} gives this local case, and Algorithm~\ref{alg:sign-muon-dist} gives the distributed AllReduce implementation in the appendix. Algorithm~\ref{alg:sign-muon-dist-1bit} and Figure~\ref{fig:schematic-allgather} give the packed AllGather implementation on page~\pageref{alg:sign-muon-dist-1bit}.

\begin{figure}[!t]
\centering
\resizebox{0.95\textwidth}{!}{%
\begin{minipage}{\textwidth}
\begin{minipage}[t]{0.49\textwidth}
\vspace{0pt}
\begin{algorithm}[H]
\caption{\textsc{SignMuon} (single worker)}
\label{alg:sign-muon}
\small
\begin{algorithmic}[1]
\STATE \textbf{Input:} stepsizes $\{\eta_t\}$, momentum $\beta\in[0,1)$, weight decay $\lambda\ge0$, NS iterations $K$, stability constant $\varepsilon>0$, scaling $\textsc{scale}\in\{\textsc{spectral},\textsc{fro}\}$, power iterations $P$.
\STATE Initialize momentum: $M_0\gets0$.
\FOR{$t=0,1,\ldots,T-1$}
\STATE Draw mini-batch; compute stochastic gradient $G_t$ at $W_t$.
\STATE Apply weight decay:\par\nobreak
\mbox{$\widetilde G_t\gets G_t+\lambda W_t$.}
\STATE Update momentum:\par\nobreak
\mbox{$M_{t+1}\gets\beta M_t+(1-\beta)\widetilde G_t$.}
\STATE Compute polar direction:\par\nobreak
{\footnotesize\mbox{$U_t\gets\mathrm{PolarNS}(M_{t+1};K,\varepsilon,\textsc{scale},P)$.}}
\STATE Extract entrywise sign:\par\nobreak
\mbox{$S_t\gets\sign(U_t)$.}
\STATE Update parameters:\par\nobreak
\mbox{$W_{t+1}\gets W_t-\eta_t S_t$.}
\ENDFOR
\end{algorithmic}
\end{algorithm}
\vspace{0.5em}
{\small\noindent
\textbf{Algorithms~\ref{alg:sign-muon}--\ref{alg:power-iter}: local SignMuon update.}
SignMuon updates momentum, approximates its polar factor with \textsc{PolarNS}, and takes an entrywise signed step. \textsc{PolarNS} scales the input and applies Newton--Schulz iterations; with spectral scaling, it calls \textsc{PowerIter} to estimate the spectral norm. All three routines run locally, without communication.\par}
\end{minipage}\hfill
\begin{minipage}[t]{0.49\textwidth}
\vspace{0pt}
\begin{algorithm}[H]
\caption{\textsc{PolarNS}: Newton-Schulz polar-factor approximation}
\label{alg:polar-ns}
\small
\begin{algorithmic}[1]
\STATE \textbf{Input:} $X\in\mathbb R^{r\times c}$, iterations $K$, $\varepsilon>0$, $\textsc{scale}\in\{\textsc{spectral},\textsc{fro}\}$, power iterations $P$.
\IF{$\textsc{scale}=\textsc{spectral}$}
\STATE Estimate $\|X\|_{\op}$:\par\nobreak
\mbox{$\sigma\gets\mathrm{PowerIter}(X,P,\varepsilon)$}
\ELSE
\STATE $\sigma\gets\|X\|_F$
\ENDIF
\STATE $Y\gets X/\max(\sigma,\varepsilon)$
\STATE $I\gets I_c$ ($c\times c$ identity)
\FOR{$k=1,\ldots,K$}
\STATE \mbox{$Y\gets\tfrac12 Y(3I-Y^\top Y)$}
\ENDFOR
\STATE \textbf{Return} $Y$
\end{algorithmic}
\end{algorithm}
\vspace{-1em}
\begin{algorithm}[H]
\caption{\textsc{PowerIter}: spectral-norm estimate}
\label{alg:power-iter}
\small
\begin{algorithmic}[1]
\STATE \textbf{Input:} $X\in\mathbb R^{r\times c}$, steps $P$, $\varepsilon>0$.
\STATE Sample $v\in\mathbb R^c$; normalize:\par\nobreak
\mbox{$v\gets v/(\|v\|_2+\varepsilon)$}
\FOR{$p=1,\ldots,P$}
\STATE \mbox{$u\gets Xv$,\quad $u\gets u/(\|u\|_2+\varepsilon)$}
\STATE \mbox{$v\gets X^\top u$,\quad $v\gets v/(\|v\|_2+\varepsilon)$}
\ENDFOR
\STATE \textbf{Return} $\|Xv\|_2$
\end{algorithmic}
\end{algorithm}
\end{minipage}
\end{minipage}%
}
\end{figure}

\begin{figure}[!t]
\centering
\resizebox{0.95\textwidth}{!}{%
\begin{minipage}{\textwidth}
\begin{minipage}[t]{0.49\textwidth}
\vspace{0pt}
\begin{algorithm}[H]
\caption{\textsc{SignMuon} (distributed with packed AllGather)}
\label{alg:sign-muon-dist-1bit}
\small
\newcommand{\stepmath}[1]{\par\nobreak\noindent\hspace*{0.25em}\(\displaystyle #1\)\par}
\begin{algorithmic}[1]
\STATE \textbf{Input:} Learning rates $\{\eta_t\}$, $M$ workers, momentum $\beta\in[0,1)$, weight decay $\lambda\ge0$, NS iterations $K$, tolerance $\varepsilon>0$, $\textsc{scale}\in\{\textsc{spectral},\textsc{fro}\}$, power iterations $P$.
\STATE Initialize each worker's momentum:
\stepmath{M_0^{(m)}\gets0,\quad m=1,\ldots,M.}
\FOR{$t=0,1,\ldots,T-1$}
\STATE \textbf{Each worker $m$ independently:}
\STATE Compute local gradient $G_t^{(m)}$ at $W_t$.
\STATE Apply weight decay:
\stepmath{\widetilde G_t^{(m)}\gets G_t^{(m)}+\lambda W_t.}
\STATE Update local moment:
\stepmath{M_{t+1}^{(m)}\gets\beta M_t^{(m)}+(1-\beta)\widetilde G_t^{(m)}.}
\STATE Compute polar decomposition:
\stepmath{\begin{aligned}U_t^{(m)}\gets{}&\mathrm{PolarNS}\bigl(M_{t+1}^{(m)};\\[-2pt]&K,\varepsilon,\textsc{scale},P\bigr).\end{aligned}}
\STATE Extract local sign:
\stepmath{S_t^{(m)}\gets\sign(U_t^{(m)})\in\{-1,+1\}^d.}
\STATE Pack signs (8 signs per byte):
\stepmath{B_t^{(m)}\gets\mathrm{PackBits}(S_t^{(m)}).}
\STATE \textbf{AllGather (per buffer):}
\STATE Gather packed \texttt{uint8} buffers:
\stepmath{\begin{aligned}&\{B_t^{(1)},\ldots,B_t^{(M)}\}\\[-2pt]&\qquad\gets\mathrm{all\_gather}(B_t^{(m)}).\end{aligned}}
\STATE \textbf{Each worker $m$ independently:}
\STATE Unpack signs for $m'=1,\ldots,M$:
\stepmath{S_t^{(m')}\gets\mathrm{UnpackBits}(B_t^{(m')}).}
\STATE Compute majority vote (ties to $+1$):
\stepmath{\bar S_t\gets\sign\!\left(\sum_{m'=1}^M S_t^{(m')}\right).}
\STATE Update parameters:
\stepmath{W_{t+1}\gets W_t-\eta_t\bar S_t.}
\ENDFOR
\STATE \textbf{Communication per worker:} send and receive $(M-1)\lceil d/8\rceil$ bytes each; total $2(M-1)\lceil d/8\rceil$ bytes per buffer.
\STATE \textbf{Encoding:} $\ge0\to1$, $<0\to0$. Decoding: $0\to-1$, $1\to+1$.
\end{algorithmic}
\end{algorithm}
\end{minipage}\hfill
\begin{minipage}[t]{0.48\textwidth}
\vspace{0pt}
\centering
\includegraphics[width=\linewidth]{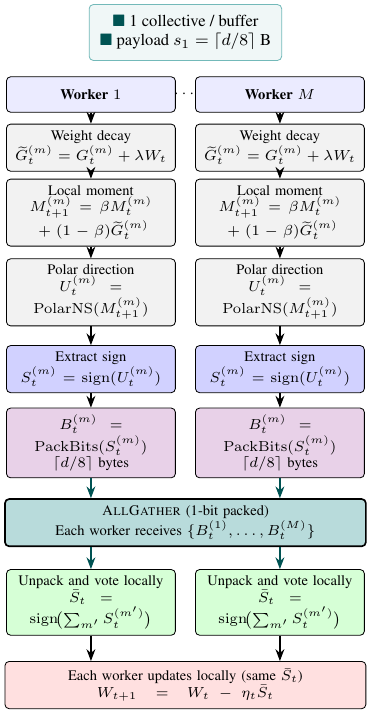}
\captionsetup{font=small}
\caption{Schematic of distributed \textsc{SignMuon} with 1-bit \textsc{AllGather} (Algorithm~\ref{alg:sign-muon-dist-1bit}). Each worker computes the same local pipeline as in the AllReduce variant, then \emph{packs} its sign vector into one bit per entry ($\lceil d/8\rceil$ bytes). A single \textsc{AllGather} disseminates every worker's packed buffer to every other worker. Each worker independently unpacks the gathered buffers and computes the majority vote $\bar S_t = \sign(\sum_{m'} S_t^{(m')})$ \emph{locally}; all workers then apply the same parameter update. The diagram shows one $d$-entry buffer; bucketed training repeats the collective and adds scalar checks.}
\label{fig:schematic-allgather}

\end{minipage}
\end{minipage}%
}
\end{figure}

\label{sec:algorithms_pseudocode}

\subsection{Normalization and local orthogonalization}

The entrywise sign operation does not preserve the singular values of $U_t^{(p)}$. We therefore distinguish the local polar direction from the final update direction. For the analysis, we use
\begin{equation}
\label{eq:signmuon_update}
D_t=\frac{\bar S_t}{\sqrt{mn}},\qquad W_{t+1}=W_t-\eta_tD_t.
\end{equation}
Since $\|\bar S_t\|_{\op}\le\|\bar S_t\|_{\F}\le\sqrt{mn}$, this choice ensures $\|D_t\|_{\op}\le1$. For one block, identical trajectories require $\eta_t^{\mathrm{theory}}=\sqrt{mn}\,\eta_t^{\mathrm{impl}}$. This conversion depends on block shape, so it does not give one global rescaling for unequal layers. Appendix~\ref{app:blocks} gives a joint bound with explicit block scales.

Computing $\polar(\bar S_t)$ after the vote would require only local work, but would change the direction. Our algorithms and stationarity bounds use the signed vote directly.

\setcounter{topnumber}{2}
\setlength{\textfloatsep}{8pt plus 2pt minus 2pt}

\subsection{Two implementations of majority vote}

The int8 implementation concatenates sign tensors, performs an elementwise integer \textsf{SUM} AllReduce, and thresholds the resulting sums. The sum lies between $-M$ and $M$, so the accumulator type must represent that range. Int8 suffices for the eight-worker experiments; larger worker counts can require a wider type. All workers receive the same sum and apply the same tie rule.

Algorithm~\ref{alg:sign-muon-dist-1bit} and Figure~\ref{fig:schematic-allgather} show the packed path. Workers start from identical parameters and zero local moments. Each computes its mini-batch gradient, adds weight decay, updates its own momentum, and applies \textsc{PolarNS}. It takes the entrywise sign and encodes nonnegative entries as one and negative entries as zero, packing eight signs per byte. AllGather then delivers every worker's packed buffer to every worker; workers call one sign collective per bucket and separate scalar collectives for checks and logging.

Each worker unpacks the buffers and counts positive signs coordinatewise. If $c_i$ of the $M$ signs are positive, their signed sum is $2c_i-M$; the majority is $+1$ when this is nonnegative and $-1$ otherwise. Workers discard padding bits. Finally, every worker subtracts the same learning-rate-scaled majority sign from its parameters, preserving synchronization while retaining its own momentum. Figure~\ref{fig:schematic-allgather}'s upper boxes are local computations, its central bar is the network exchange, and its lower boxes perform local voting and updating. AllReduce gives the same vote for identical input signs, representable sums, and the same tie rule; its schematic is in Appendix~\ref{app:implementation}.

\let\savedFigureBarrier\FloatBarrier
\let\FloatBarrier\relax
\section{SignMuon theory and communication}
\label{sec:analysis}
\let\FloatBarrier\savedFigureBarrier

\subsection{Alignment of the momentum-polar vote}

Let $g_t=\nabla f(W_t)$ for one matrix variable with $d=mn$ entries. Assume continuously differentiable $f(W)\ge f^\ast> -\infty$, fix $W_0$, and use deterministic $\eta_t>0$. Spectral-norm smoothness means
\begin{equation}
\|\nabla f(W)-\nabla f(W')\|_\ast\le L_\ast\|W-W'\|_{\op},
\end{equation}
where $\|\cdot\|_\ast$ is the nuclear norm. For SignMuon's actual vote $\bar S_t$ and normalized direction $D_t=\bar S_t/\sqrt d$, the descent inequality is
\begin{equation}
f(W_{t+1})\le f(W_t)-\eta_t\langle g_t,D_t\rangle+\frac{L_\ast\eta_t^2}{2}.
\end{equation}
Appendix~\ref{app:blocks} extends this statement to jointly updated blocks and the unnormalized implementation. Spectral smoothness alone does not improve the curvature constant for a normalized sign step (Appendix~\ref{app:smoothness_comparison}).

Define the error probability of the final vote and its gradient-weighted loss by
\begin{equation}
q_{ij,t}=\Pr(\bar S_{t,ij}\ne\sign(g_{t,ij})\mid W_t),\qquad
e_t=\frac1{\sqrt d}\sum_{ij}|g_{t,ij}|q_{ij,t}.
\end{equation}
These probabilities describe the full momentum--polar--sign pipeline, including aggregation. They are not error probabilities for signs of raw stochastic gradients. Lemma~\ref{lem:inner_product_identity} gives
$\E[\langle g_t,D_t\rangle\mid W_t]\ge V_t-2e_t$, where $V_t=\|g_t\|_1/\sqrt d$.

\begin{theorem}[Conditional SignMuon stationarity]\label{thm:signmuon_rate}
Let $A_T=\sum_{t<T}\eta_t$, $B_T=\sum_{t<T}\eta_t^2$, and $\Delta=f(W_0)-f^\ast$. Suppose the actual SignMuon vote satisfies
$2\E e_t\le\rho\E V_t+\delta_t$ for fixed $0\le\rho<1$ and deterministic $\delta_t\ge0$. Then
\begin{equation}\label{eq:l1_proxy_main}
\frac1{A_T}\sum_{t<T}\eta_t\E V_t
\le\frac1{1-\rho}\left[\frac{\Delta}{A_T}
+\frac{L_\ast B_T}{2A_T}+\frac{\sum_{t<T}\eta_t\delta_t}{A_T}\right].
\end{equation}
\end{theorem}

The proof in Appendix~\ref{sec:main-convergence} sums the descent inequalities and absorbs the relative sign-error term. For $\eta_t=\eta_0/\sqrt T$, a residual average of order $T^{-1/2}$ gives the same stationarity order. Along an infinite run, weighted average stationarity follows if $A_T\to\infty$, $B_T/A_T\to0$, and the weighted residual average vanishes. A nonvanishing residual leaves a bound on average stationarity. These statements concern expected gradient norms, not zero training error or convergence of the last iterate.

\subsection{Momentum and polar-solver error}

A second result analyzes fixed-batch SignMuon directly through momentum tracking. Let $\bar M_{t+1}=M^{-1}\sum_pM_{t+1}^{(p)}$ and suppose the final normalized vote satisfies
\begin{equation}\label{eq:alignment_main}
\E\langle \bar M_{t+1},D_t\rangle\ge\kappa\E\|\bar M_{t+1}\|_\ast-\varepsilon_t,
\qquad 0<\kappa\le1.
\end{equation}
Under conditionally unbiased gradients with bounded variance, Theorem~\ref{thm:momentum_signmuon} gives $T^{-1}\sum_{t<T}\E\|g_t\|_\ast=O(T^{-1/4})$ when $\eta=\eta_0T^{-3/4}$, $\beta=1-T^{-1/2}$, and $T^{-1}\sum_{t<T}\varepsilon_t=O(T^{-1/4})$, with $\kappa$ independent of $T$. This adapts the momentum-tracking argument of \citet{qian2026degenerate} to the final SignMuon direction. Appendix~\ref{app:joint-momentum} extends the rate to unnormalized blocks and biased inputs. Alignment remains additional; clipping, fixed momentum, and experimental schedules do not automatically satisfy the rate assumptions.

The local numerical error and the signing loss play different roles. For momentum $X$, $\|U\|_{\op}\le1$, and $D=\sign(U)/\sqrt d$, let $[a]_+=\max(a,0)$ and define
\begin{equation}
\epsilon_{\mathrm{NS}}=\|X\|_\ast-\langle X,U\rangle,\qquad
\chi_{\mathrm{sign}}=[\langle X,U-D\rangle]_+.
\end{equation}
Proposition~\ref{prop:signmuon_gap} shows
$\langle g,D\rangle\ge\|g\|_\ast-2\|g-X\|_\ast-\epsilon_{\mathrm{NS}}-\chi_{\mathrm{sign}}$.
For rank-$r$ momentum and cubic Newton--Schulz scaled by $\|X\|_{\F}$, Theorem~\ref{thm:newton_schulz} gives $\epsilon_{\mathrm{NS}}\le r(2/3)^K\|X\|_\ast$, even without a lower bound on the least positive singular value. A scaling floor adds a tolerance term. These bounds control numerical inexactness, but not the loss from signing or majority voting.

\subsection{Payload size and collective cost}

We use the latency--bandwidth model $T(s)=\alpha+\beta_{\mathrm{net}}s$ for a message of $s$ bytes. Here $\beta_{\mathrm{net}}$ denotes inverse bandwidth, distinct from the momentum coefficient. Related communication models appear in distributed numerical solvers \citep{kumar2014communication,kumar2015multilevel,kumar2013high}. For $d$ communicated entries, an int8 buffer has $s_8=d$ bytes and a packed binary buffer has $s_1=\lceil d/8\rceil$ bytes.

For a ring with $M>1$ workers, the two collective costs are
\begin{align}
T_{\mathrm{AR}}&\approx2(M-1)\alpha+2(1-1/M)d\,\beta_{\mathrm{net}},\\
T_{\mathrm{AG}}&\approx(M-1)\alpha+(M-1)\lceil d/8\rceil\,\beta_{\mathrm{net}}.
\end{align}
These expressions model communication, not the full training iteration. Int8 AllReduce reduces the bandwidth term by a factor of four relative to float32 AllReduce with the same collective. Packing reduces the local representation by a factor of 32 relative to float32, but AllGather replicates that representation. Ignoring byte padding, its bandwidth term divided by that of int8 AllReduce is $M/16$. Thus packed AllGather has a smaller modeled bandwidth term below 16 workers and a comparable term at 16.

AllGather also requires a buffer of $M\lceil d/8\rceil$ bytes and local vote computation. AllReduce returns one summed buffer and avoids that replication. Neither collective communicates the polar factors. Appendix~\ref{app:comm} derives the costs, including other collective schedules and parameter-server variants.

\FloatBarrier

\setcounter{topnumber}{2}
\setlength{\textfloatsep}{8pt plus 2pt minus 2pt}
\section{Experiments}
\label{sec:experiments}
\label{sec:exp_setup_common}

\subsection{Pretraining across architectures and corpora}
\label{sec:pretraining}

We train GPT-2 Large, GPT-2 XL, and OPT-350M from random initialization using their native model configurations and tokenizers. We compare SignMuon, SignSGD, SignAdam, and SignLion on TinyStories, FineMath-4+, and FineWeb-Edu (\texttt{sample-10BT}).

\paragraph{Experimental setup.}
Pretraining comparisons use eight RTX 5090 or Quadro RTX 6000 GPUs, three training seeds, and a 30-minute wall-clock budget per run, including training, validation, and communication. All methods use packed AllGather and are compared within the same architecture and GPU configuration. Appendix~\ref{app:pretraining-details} provides the full setup. NanoGPT weak-scaling experiments hold work per GPU fixed across one to 16 RTX 4090 GPUs, while additional nanoGPT pretraining experiments use a five-minute budget. Appendix~\ref{app:nanogpt-scaling-details} describes the normalization and weak-scaling settings; Appendix~\ref{sec:nanogpt-setup} details the nanoGPT pretraining setup.

\paragraph{Evaluation metrics.}
Validation log-perplexity is mean next-token cross-entropy in nats per token; lower is better. Figure~\ref{fig:pretraining-single-tokens} plots single-seed validation loss against cumulative training tokens processed. Figure~\ref{fig:pretraining-multi-time} plots mean validation loss against elapsed time, with bands showing $\pm$ one sample standard deviation across three seeds. For each optimizer, we linearly interpolate the seed trajectories only within their shared observed time interval. Table~\ref{tab:pretraining-combined} reports, for each method, the mean and sample variance across three seeds, using each run's final logged validation loss.

\begingroup
\setlength{\intextsep}{6pt}
\captionsetup[table]{skip=4pt}

\begin{table}[!t]
\centering
\small
\caption{Combined pretraining results: mean final validation log-perplexity (sample variance), three seeds. SignMuon has the lowest mean in all nine dataset--model pairs (bold).}
\label{tab:pretraining-combined}
\begin{tabular}{@{}lccc@{}}
\toprule
Optimizer & GPT-2 Large & GPT-2 XL & OPT-350M \\
\midrule
\multicolumn{4}{l}{\textbf{TinyStories}}\\[2pt]
SignMuon & $\mathbf{2.7687}\;(1.0\times10^{-4})$ & $\mathbf{2.2546}\;(7.7\times10^{-5})$ & $\mathbf{4.2259}\;(2.7\times10^{-4})$ \\
SignSGD & $4.4749\;(1.4\times10^{-4})$ & $4.5806\;(2.2\times10^{-4})$ & $4.6765\;(1.5\times10^{-3})$ \\
SignAdam & $3.3312\;(4.4\times10^{-3})$ & $3.5693\;(1.4\times10^{-5})$ & $5.9780\;(5.3\times10^{-9})$ \\
SignLion & $3.9634\;(5.1\times10^{-3})$ & $3.8535\;(5.1\times10^{-3})$ & $5.9732\;(2.7\times10^{-8})$ \\
\midrule
\multicolumn{4}{l}{\textbf{FineMath-4+}}\\[2pt]
SignMuon & $\mathbf{4.9500}\;(4.0\times10^{-5})$ & $\mathbf{4.4443}\;(1.3\times10^{-4})$ & $\mathbf{6.0110}\;(8.5\times10^{-3})$ \\
SignSGD & $6.4434\;(3.2\times10^{-4})$ & $6.5380\;(1.1\times10^{-3})$ & $6.7838\;(4.1\times10^{-4})$ \\
SignAdam & $5.3616\;(1.3\times10^{-2})$ & $5.6496\;(1.4\times10^{-5})$ & $7.1280\;(6.7\times10^{-6})$ \\
SignLion & $5.5795\;(3.4\times10^{-3})$ & $5.6571\;(5.8\times10^{-3})$ & $7.1168\;(4.3\times10^{-6})$ \\
\midrule
\multicolumn{4}{l}{\textbf{FineWeb-Edu}}\\[2pt]
SignMuon & $\mathbf{6.0239}\;(7.4\times10^{-5})$ & $\mathbf{5.5985}\;(2.6\times10^{-5})$ & $\mathbf{6.8540}\;(4.9\times10^{-4})$ \\
SignSGD & $7.1856\;(2.9\times10^{-4})$ & $7.2327\;(3.0\times10^{-4})$ & $7.3340\;(3.6\times10^{-4})$ \\
SignAdam & $6.2894\;(4.9\times10^{-3})$ & $6.5859\;(2.2\times10^{-5})$ & $7.6522\;(9.9\times10^{-8})$ \\
SignLion & $6.4783\;(3.6\times10^{-4})$ & $6.5896\;(3.3\times10^{-3})$ & $7.6505\;(2.2\times10^{-7})$ \\
\bottomrule
\end{tabular}

\end{table}

\endgroup

\subsection{Experimental results}
\label{sec:experimental-scope}

\paragraph{Token efficiency.}
Figure~\ref{fig:pretraining-single-tokens} shows sustained gains in token efficiency across the evaluated architectures and datasets. For GPT-2 Large and GPT-2 XL, SignMuon achieves lower validation loss than the other signed optimizers at every shared evaluation checkpoint after initialization. For OPT-350M, it establishes a consistent advantage after approximately 23 to 26 million training tokens, depending on the dataset, and maintains that advantage throughout the remaining shared token range. At the final shared token count, SignMuon achieves the lowest validation loss across all nine combinations of architecture and dataset. On TinyStories with GPT-2 XL, SignMuon achieves better validation performance with one-third as many training tokens: a loss of 3.6554 after 6.5536M tokens, compared with SignAdam's 3.7429 after 19.6608M tokens.

\begin{figure}[H]
\centering
\includegraphics[width=0.99\linewidth]{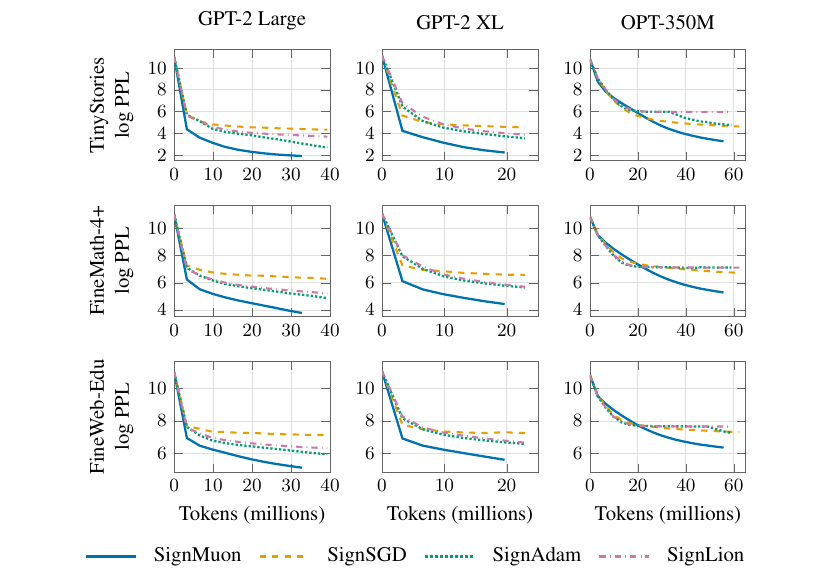}
\par\smallskip

\caption{Single-seed pretraining token efficiency. Rows identify datasets and columns identify architectures. Curves show raw validation log-perplexity, including initialization, without running-minimum smoothing. SignMuon has the lowest loss at each panel's last shared token count. Endpoints differ across methods; matched-token comparisons use only the overlap of the curves.}
\label{fig:pretraining-single-tokens}
\end{figure}

\paragraph{Performance under a fixed time budget.}
Under the same wall-clock budget, SignMuon achieves the lowest mean validation loss across all nine settings (Figure~\ref{fig:pretraining-multi-time}). On FineMath-4+ with GPT-2 XL, it reaches a mean validation loss of 5.7 in less than one-third of the time taken by SignAdam: approximately 8.0 minutes compared with 25.4 minutes. Table~\ref{tab:pretraining-combined} further shows the lowest mean final validation loss for SignMuon in every setting. For GPT-2 XL, SignMuon also uses half the persistent optimizer memory of SignAdam (5.80 versus 11.61 GiB per GPU) and reduces peak training VRAM from 26.96 to 21.03 GiB (Figure~\ref{fig:memory-xl}). Appendix~\ref{app:memory-details} provides the memory measurement protocol and results for all three model sizes.

\begin{figure}[H]
\centering
\includegraphics[width=\linewidth]{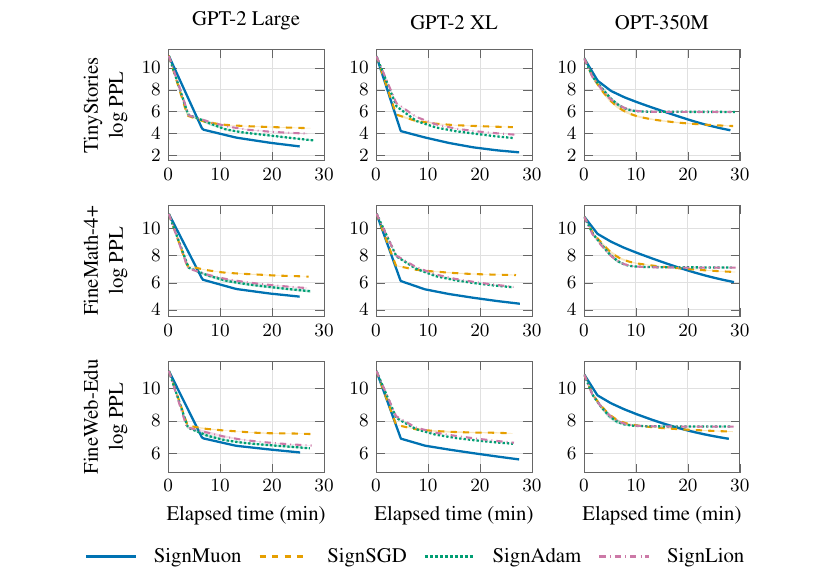}
\par\smallskip

\caption{Multi-seed pretraining wall-clock performance. Lines show mean validation log-perplexity; bands are $\pm$ one sample standard deviation over three seeds. SignMuon is lowest at each panel's latest shared time. Interpolation uses only time ranges observed by all three seeds; no curve is extrapolated. Table~\ref{tab:pretraining-combined} uses per-seed last evaluations. Timings are compared only within each architecture's campaign.}
\label{fig:pretraining-multi-time}
\end{figure}

\begin{figure}[H]
\centering
\begin{minipage}[t]{0.38\textwidth}
\vspace{0pt}
\centering
\includegraphics[width=\linewidth]{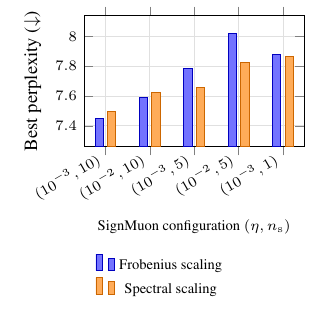}
\par\vspace{6pt}

\par
\captionsetup{font=small}
\caption{nanoGPT scaling comparison across learning rates $\eta$ and Newton--Schulz counts $n_{\mathrm{s}}$. Paired best perplexities differ by less than 2.5\%; neither scaling wins in every configuration. Both yield very similar perplexities; we use only Frobenius scaling in the main pretraining experiments.}
\label{fig:nanogpt_scaling_compare}
\end{minipage}\hfill
\begin{minipage}[t]{0.60\textwidth}
\vspace{0pt}
\centering
\includegraphics[width=0.82\linewidth]{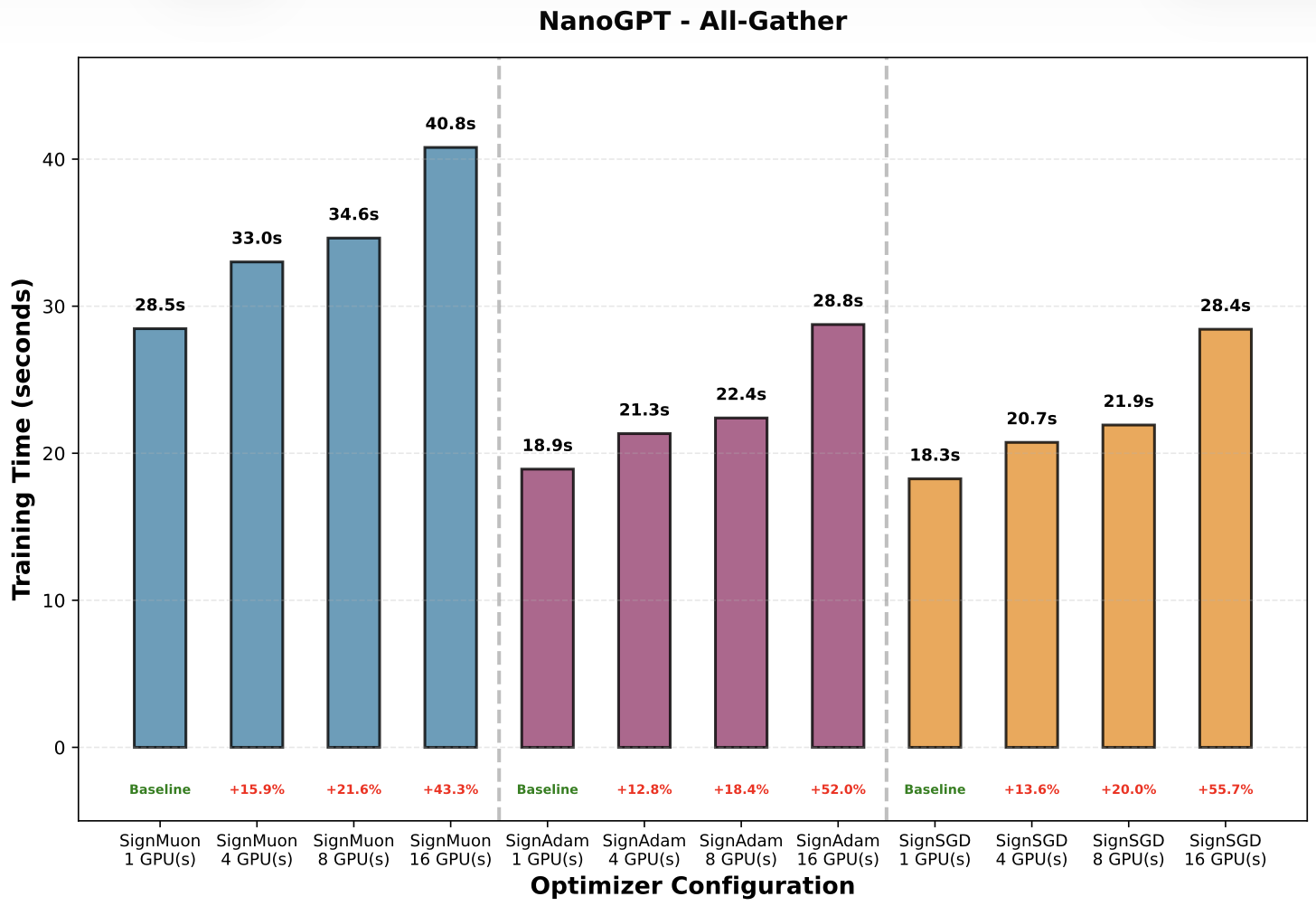}
\captionsetup{font=small}
\caption{nanoGPT AllGather weak scaling with 1, 4, 8, and 16 GPUs. Labels report training time and the increase over each optimizer's one-GPU baseline. SignMuon has a higher absolute runtime in this weak-scaling experiment. The separate pretraining comparison in Figure~\ref{fig:pretraining-multi-time} reports validation loss against elapsed time.}
\label{fig:weak-scaling-nanogpt}
\end{minipage}
\end{figure}

\paragraph{Sensitivity to polar normalization.}
In the separate nanoGPT study, Frobenius and spectral scaling yield best perplexities within 2.5\% of each other across the five paired settings in Figure~\ref{fig:nanogpt_scaling_compare}. Frobenius scaling yields lower perplexity at ten Newton--Schulz iterations, while spectral scaling yields lower perplexity at five and one. The main pretraining experiments use Frobenius scaling with ten iterations.

\begin{figure}[H]\centering

\begin{minipage}{0.49\linewidth}\centering\includegraphics[width=\linewidth]{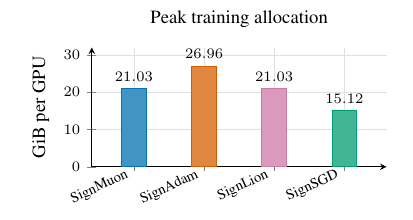}\end{minipage}\hfill
\begin{minipage}{0.49\linewidth}\centering\includegraphics[width=\linewidth]{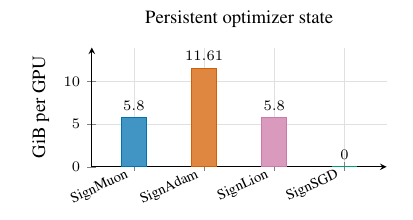}\end{minipage}
\caption{GPT-2 XL memory per GPU on eight RTX 5090 GPUs (TinyStories, seed 1337). SignMuon uses less peak memory than SignAdam and half its optimizer state. Peaks include temporary allocations; state counts optimizer tensors only.}
\label{fig:memory-xl}
\end{figure}

\paragraph{Distributed weak scaling.}
Runtime increases by 43.3\% for SignMuon, compared with 52.0\% for SignAdam and 55.7\% for SignSGD (Figure~\ref{fig:weak-scaling-nanogpt}). SignMuon has the smallest relative increase but a higher absolute runtime. In the additional nanoGPT pretraining experiments, it achieves the lowest mean validation loss on all three datasets (Table~\ref{tab:final}; Appendix~\ref{sec:nanogpt-results}).

\FloatBarrier

\section{Conclusion}
\label{sec:conclusion}
SignMuon combines local polar-direction computation with packed 1-bit majority-vote communication. Across the evaluated architectures and datasets, it improves token efficiency and achieves lower validation loss within the same wall-clock budget. These results demonstrate the value of matrix-based updates for learning efficiently while communicating only packed signs. Our theory complements these findings with conditional stationarity guarantees that separate momentum, polar-solver, and signing errors.

\paragraph{Reproducibility statement.}
Algorithms~\ref{alg:sign-muon} through~\ref{alg:sign-muon-dist-1bit} specify the local optimization and distributed aggregation routines. Appendix~\ref{app:pretraining-details} documents data preparation, optimizer hyperparameters, numerical precision, training seeds, and evaluation procedures. Appendix~\ref{sec:nanogpt-setup} gives the separate nanoGPT experimental protocol. Appendices~\ref{app:comm} and~\ref{sec:main-convergence} provide the communication analysis and proofs of the theoretical results. 

\paragraph{AI use statement.}
We used generative AI to refine the prose, improve clarity and readability, assist with the organization and presentation of the paper, and provide coding assistance. We take responsibility for the final manuscript, including its theoretical claims and experimental results.

\label{page:main-end}
\label{page:references-start}
\bibliographystyle{styles/iclr2027_conference}
\bibliography{refs}

\begin{thebibliography}{27}
\providecommand{\natexlab}[1]{#1}
\providecommand{\url}[1]{\texttt{#1}}
\expandafter\ifx\csname urlstyle\endcsname\relax
  \providecommand{\doi}[1]{doi: #1}\else
  \providecommand{\doi}{doi: \begingroup \urlstyle{rm}\Url}\fi

\bibitem[Alistarh et~al.(2017)Alistarh, Grubic, Li, Tomioka, and Vojnovic]{alistarh2017qsgd}
Dan Alistarh, Demjan Grubic, Jerry Li, Ryota Tomioka, and Milan Vojnovic.
\newblock {QSGD}: Communication-efficient {SGD} via gradient quantization and encoding.
\newblock In \emph{Advances in Neural Information Processing Systems}, 2017.

\bibitem[Bernstein et~al.(2018)Bernstein, Zhao, Azizzadenesheli, and Anandkumar]{bernstein2018signsgd}
Jeremy Bernstein, Jiawei Zhao, Kamyar Azizzadenesheli, and Anima Anandkumar.
\newblock {signSGD} with majority vote is communication efficient and fault tolerant.
\newblock \emph{CoRR}, abs/1810.05291, 2018.
\newblock URL \url{https://arxiv.org/abs/1810.05291}.

\bibitem[Bj{\"o}rck \& Bowie(1971)Bj{\"o}rck and Bowie]{bowie1971}
{\AA}ke Bj{\"o}rck and Clyde Bowie.
\newblock An iterative algorithm for computing the best estimate of an orthogonal matrix.
\newblock \emph{SIAM Journal on Numerical Analysis}, 8\penalty0 (2):\penalty0 358--364, 1971.
\newblock \doi{10.1137/0708036}.

\bibitem[Chen et~al.(2023)Chen, Liang, Huang, Real, Wang, Pham, Dong, Luong, Hsieh, Lu, and Le]{chen2023symbolic}
Xiangning Chen, Chen Liang, Da~Huang, Esteban Real, Kaiyuan Wang, Hieu Pham, Xuanyi Dong, Thang Luong, Cho{-}Jui Hsieh, Yifeng Lu, and Quoc~V. Le.
\newblock Symbolic discovery of optimization algorithms.
\newblock In \emph{Advances in Neural Information Processing Systems 36}, 2023.
\newblock URL \url{http://papers.nips.cc/paper_files/paper/2023/hash/9a39b4925e35cf447ccba8757137d84f-Abstract-Conference.html}.

\bibitem[Gupta et~al.(2018)Gupta, Koren, and Singer]{gupta2018shampoopreconditionedstochastictensor}
Vineet Gupta, Tomer Koren, and Yoram Singer.
\newblock Shampoo: Preconditioned stochastic tensor optimization, 2018.
\newblock URL \url{https://arxiv.org/abs/1802.09568}.
\newblock arXiv:1802.09568.

\bibitem[Higham(1986)]{article}
Nicholas~J. Higham.
\newblock Computing the polar decomposition---with applications.
\newblock \emph{SIAM Journal on Scientific and Statistical Computing}, 7\penalty0 (4):\penalty0 1160--1174, 1986.

\bibitem[Hoefler et~al.(2010)Hoefler, Schneider, and Lumsdaine]{hoefler2010characterizing}
Torsten Hoefler, Timo Schneider, and Andrew Lumsdaine.
\newblock Characterizing the influence of system noise on large-scale applications by simulation.
\newblock In \emph{Proceedings of the 2010 ACM/IEEE International Conference for High Performance Computing, Networking, Storage and Analysis (SC)}, pp.\  1--11, 2010.

\bibitem[Jiang \& Zhang(2025)Jiang and Zhang]{jiang2025lion}
Wei Jiang and Lijun Zhang.
\newblock Convergence analysis of the lion optimizer in centralized and distributed settings.
\newblock \emph{arXiv preprint arXiv:2508.12327}, 2025.
\newblock URL \url{https://arxiv.org/abs/2508.12327}.

\bibitem[Jordan(2024)]{jordan6muon}
Keller Jordan.
\newblock Muon: An optimizer for hidden layers in neural networks.
\newblock \url{https://kellerjordan.github.io/posts/muon/}, 2024.
\newblock Blog post. Accessed: 2025-12-26.

\bibitem[Kumar(2014)]{kumar2014communication}
Pawan Kumar.
\newblock Communication optimal least squares solver.
\newblock In \emph{2014 IEEE Intl Conf on High Performance Computing and Communications, 2014 IEEE 6th Intl Symp on Cyberspace Safety and Security, 2014 IEEE 11th Intl Conf on Embedded Software and Syst (HPCC,CSS,ICESS)}, pp.\  316--319, 2014.
\newblock \doi{10.1109/HPCC.2014.55}.

\bibitem[Kumar(2015)]{kumar2015multilevel}
Pawan Kumar.
\newblock Multilevel communication optimal least squares.
\newblock \emph{Procedia Computer Science}, 51:\penalty0 1838--1847, 2015.
\newblock ISSN 1877-0509.
\newblock \doi{10.1016/j.procs.2015.05.410}.
\newblock URL \url{https://www.sciencedirect.com/science/article/pii/S1877050915012181}.
\newblock International Conference On Computational Science, ICCS 2015.

\bibitem[Kumar et~al.(2013)Kumar, Markidis, Lapenta, Meerbergen, and Roose]{kumar2013high}
Pawan Kumar, Stefano Markidis, Giovanni Lapenta, Karl Meerbergen, and Dirk Roose.
\newblock High performance solvers for implicit particle in cell simulation.
\newblock \emph{Procedia Computer Science}, 18:\penalty0 2251--2258, 2013.
\newblock ISSN 1877-0509.
\newblock \doi{10.1016/j.procs.2013.05.396}.
\newblock URL \url{https://www.sciencedirect.com/science/article/pii/S1877050913005395}.
\newblock 2013 International Conference on Computational Science.

\bibitem[Lin et~al.(2017)Lin, Han, Mao, Wang, and Dally]{lin2017deep}
Yujun Lin, Song Han, Huizi Mao, Yu~Wang, and William~J. Dally.
\newblock Deep gradient compression: Reducing the communication bandwidth for distributed training.
\newblock \emph{CoRR}, abs/1712.01887, 2017.
\newblock URL \url{https://arxiv.org/abs/1712.01887}.

\bibitem[Liu et~al.(2024)Liu, Wu, Chen, Liang, Zhu, Liang, Krishnamoorthi, and Liu]{liu2024distributedlion}
Bo~Liu, Lemeng Wu, Lizhang Chen, Kaizhao Liang, Jiaxu Zhu, Chen Liang, Raghuraman Krishnamoorthi, and Qiang Liu.
\newblock Communication efficient distributed training with distributed lion.
\newblock In \emph{Advances in Neural Information Processing Systems}, volume~37, 2024.
\newblock URL \url{https://proceedings.neurips.cc/paper_files/paper/2024/hash/20cea6c1b36ae5f69c48427a68b67fbc-Abstract-Conference.html}.

\bibitem[Loshchilov \& Hutter(2017)Loshchilov and Hutter]{loshchilov2017decoupled}
Ilya Loshchilov and Frank Hutter.
\newblock Decoupled weight decay regularization.
\newblock \emph{CoRR}, abs/1711.05101, 2017.
\newblock URL \url{https://arxiv.org/abs/1711.05101}.

\bibitem[Martens \& Grosse(2015)Martens and Grosse]{martens2015optimizing}
James Martens and Roger Grosse.
\newblock Optimizing neural networks with kronecker-factored approximate curvature.
\newblock In \emph{Proceedings of the 32nd International Conference on Machine Learning}, volume~37 of \emph{Proceedings of Machine Learning Research}, pp.\  2408--2417. PMLR, 2015.

\bibitem[Polyak(1964)]{POLYAK19641}
B.~T. Polyak.
\newblock Some methods of speeding up the convergence of iteration methods.
\newblock \emph{USSR Computational Mathematics and Mathematical Physics}, 4\penalty0 (5):\penalty0 1--17, 1964.
\newblock \doi{10.1016/0041-5553(64)90137-5}.

\bibitem[Qian \& Richt{\'a}rik(2026)Qian and Richt{\'a}rik]{qian2026degenerate}
Xun Qian and Peter Richt{\'a}rik.
\newblock Convergence analysis of {Muon}-type methods with inexact {LMO} in the degenerate case.
\newblock \emph{arXiv preprint arXiv:2606.21581}, 2026.
\newblock URL \url{https://arxiv.org/abs/2606.21581}.

\bibitem[Rabenseifner(2004)]{rabenseifner2004allreduce}
Rolf Rabenseifner.
\newblock Optimization of collective reduction operations.
\newblock In \emph{International Conference on Computational Science (ICCS)}, pp.\  1--9, 2004.

\bibitem[Shen et~al.(2025)Shen, Huang, Huang, Shen, and Zhang]{shen2025muon}
Wei Shen, Ruichuan Huang, Minhui Huang, Cong Shen, and Jiawei Zhang.
\newblock On the convergence analysis of muon, 2025.
\newblock URL \url{https://arxiv.org/abs/2505.23737}.
\newblock arXiv:2505.23737.

\bibitem[Stich et~al.(2018)Stich, Cordonnier, and Jaggi]{stich2018sparsified}
Sebastian~U Stich, Jean-Baptiste Cordonnier, and Martin Jaggi.
\newblock Sparsified sgd with memory.
\newblock \emph{Advances in neural information processing systems}, 31, 2018.

\bibitem[Sun et~al.(2023)Sun, Wang, Li, and Wang]{sun2023momentum}
Tao Sun, Qingsong Wang, Dongsheng Li, and Bao Wang.
\newblock Momentum ensures convergence of {SIGNSGD} under weaker assumptions.
\newblock In \emph{Proceedings of the 40th International Conference on Machine Learning}, volume 202 of \emph{Proceedings of Machine Learning Research}, pp.\  33077--33099. PMLR, 2023.
\newblock URL \url{https://proceedings.mlr.press/v202/sun23l.html}.

\bibitem[Sutskever et~al.(2013)Sutskever, Martens, Dahl, and Hinton]{pmlr-v28-sutskever13}
Ilya Sutskever, James Martens, George Dahl, and Geoffrey Hinton.
\newblock On the importance of initialization and momentum in deep learning.
\newblock In Sanjoy Dasgupta and David McAllester (eds.), \emph{Proceedings of the 30th International Conference on Machine Learning}, volume~28 of \emph{Proceedings of Machine Learning Research}, pp.\  1139--1147. PMLR, 2013.

\bibitem[Thakur et~al.(2005)Thakur, Rabenseifner, and Gropp]{thakur2005optimization}
Rajeev Thakur, Rolf Rabenseifner, and William Gropp.
\newblock Optimization of collective communication operations in {MPICH}.
\newblock \emph{International Journal of High Performance Computing Applications}, 19\penalty0 (1):\penalty0 49--66, 2005.

\bibitem[Vogels et~al.(2019)Vogels, Karimireddy, and Jaggi]{vogels2019powersgd}
Thijs Vogels, Sai~Praneeth Karimireddy, and Martin Jaggi.
\newblock {PowerSGD}: Practical low-rank gradient compression for distributed optimization.
\newblock In \emph{Advances in Neural Information Processing Systems}, 2019.

\bibitem[Wang et~al.(2026)Wang, Shen, and Zhang]{wang2026olionapproachinghadamardideal}
Zixiao Wang, Yifei Shen, and Huishuai Zhang.
\newblock Olion: Approaching the hadamard ideal by intersecting spectral and $\ell_{\infty}$ implicit biases, 2026.
\newblock URL \url{https://arxiv.org/abs/2602.01105}.

\bibitem[Wen et~al.(2017)Wen, Xu, Yan, Wu, Wang, Chen, and Li]{wen2017terngrad}
Wei Wen, Cong Xu, Feng Yan, Chunpeng Wu, Yandan Wang, Yiran Chen, and Hai Li.
\newblock Terngrad: Ternary gradients to reduce communication in distributed deep learning.
\newblock \emph{Advances in neural information processing systems}, 30, 2017.

\end{thebibliography}
\clearpage
\appendix
\label{page:appendix-start}

\section{Implementation details}
\label{app:implementation}

The main text gives the single-worker and packed AllGather algorithms, the local polar and norm-estimation routines, and the packed communication schematic. This appendix gives the int8 AllReduce algorithm and schematic, with further implementation details. All workers start from the same parameters and apply the same aggregate, while retaining separate local momentum buffers.

\subsection{AllReduce implementation}

Algorithm~\ref{alg:sign-muon-dist} gives the integer-AllReduce implementation illustrated in Figure~\ref{fig:schematic-allreduce}. Each worker computes its gradient, momentum, and polar direction locally, then contributes an integer sign buffer to a single sum reduction. Thresholding the returned sums produces the same majority vote on every worker, so the parameter update needs no further communication. Unlike the packed AllGather path in Algorithm~\ref{alg:sign-muon-dist-1bit}, this implementation receives one aggregate buffer rather than every worker's input; the accumulator must represent all possible sums without overflow.

\begin{algorithm}[t!]
\caption{\textsc{SignMuon} (distributed with all-reduce)}
\label{alg:sign-muon-dist}
\begin{algorithmic}[1]
\STATE \textbf{Input:} Learning rates $\{\eta_t\}$, $M$ workers, momentum $\beta\in[0,1)$, weight decay $\lambda\ge 0$, NS iterations $K$, tolerance $\varepsilon>0$, scale $\in\{\textsc{spectral},\textsc{fro}\}$, power iterations $P$
\STATE \textbf{Initialize:} $M_0^{(m)} \gets 0$ for all workers $m \in \{1,\ldots,M\}$
\FOR{$t=0,1,\ldots,T-1$}
  \STATE \textbf{Each worker $m$ independently:}
  \STATE \quad Compute local gradient $G_t^{(m)}$ at current parameters $W_t$
  \STATE \quad Apply weight decay: $\widetilde{G}_t^{(m)} \gets G_t^{(m)} + \lambda W_t$
  \STATE \quad Update local momentum: $M_{t+1}^{(m)} \gets \beta M_t^{(m)} + (1-\beta)\widetilde{G}_t^{(m)}$
  \STATE \quad Compute polar decomposition: $U_t^{(m)} \gets \mathrm{PolarNS}(M_{t+1}^{(m)};K,\varepsilon,\textsc{scale},P)$
  \STATE \quad Extract local sign: $S_t^{(m)} \gets \sign(U_t^{(m)}) \in \{-1,+1\}^d$
  \STATE \textbf{All-reduce (collective):}
  \STATE \quad All workers participate in sum-reduce:
  \STATE \quad \quad $\mathrm{sum\_signs} \gets \sum_{m=1}^M S_t^{(m)}$ \hfill \textcolor{gray}{\textit{(all-reduce with SUM, int8)}}
  \STATE \quad All workers compute majority vote locally:
  \STATE \quad \quad $\bar{S}_t \gets \sign(\mathrm{sum\_signs}) \in \{-1,+1\}^d$ \hfill \textcolor{gray}{\textit{(ties default to +1)}}
  \STATE \textbf{Each worker $m$ independently:}
  \STATE \quad Update parameters: $W_{t+1} \gets W_t - \eta_t \bar{S}_t$
\ENDFOR
\STATE \textbf{Local payload:} $d$ bytes (int8). Appendix~\ref{app:comm} gives total traffic.
\end{algorithmic}
\end{algorithm}

\subsection{Packed majority vote}

Algorithm~\ref{alg:sign-muon-dist-1bit} and Figure~\ref{fig:schematic-allgather} implement majority voting through packed AllGather. The encoding below explains how each worker recovers the same integer sign sum locally from the received bits, including the tie rule and the treatment of padding.

The displayed algorithms treat one matrix or concatenated buffer. The recorded implementation uses one sign AllGather per bucket and separate scalar collectives for checks and reporting (Appendix~\ref{app:pretraining-details}). Each worker encodes a nonnegative entry as one and a negative entry as zero, packs eight entries per byte, and AllGathers the packed buffer. It then counts the positive signs for each coordinate. If that count is $c_i$, the signed sum is $2c_i-M$. The worker applies the sign of this sum, with ties mapped to $+1$. Byte padding does not represent additional parameter entries and must not enter the vote.

\subsection{Local computation cost}
For an $r\times c$ matrix, Algorithm~\ref{alg:polar-ns} forms $Y^\top Y$ and multiplies by that $c\times c$ Gram matrix. Each Newton--Schulz iteration costs $O(rc^2)$ arithmetic operations with classical dense multiplication. With $K$ such iterations and $P$ power iterations, the local cost is $O(Krc^2+Prc)$; extracting signs costs $O(rc)$. The Gram matrix needs $O(c^2)$ temporary storage in addition to matrix-sized buffers. A thin SVD costs $O(rc\min\{r,c\})$. Applying Newton--Schulz to the transpose when $c>r$ can reduce its cost to $O(Krc\min\{r,c\})$, but the displayed routine uses the right Gram matrix as written. GPU runtime also depends on the hardware and kernel implementation.

SignMuon therefore adds local matrix products relative to signSGD or Lion. Compression helps end-to-end time only when communication savings and optimization progress compensate for that work. Persistent optimizer state excludes temporary Gram and multiplication buffers. Figure~\ref{fig:memory-xl} and Table~\ref{tab:memory-models} distinguish this stored state from peak training allocation, which also includes temporary workspace.

\subsection{Local polar approximation}

Algorithm~\ref{alg:polar-ns} gives the \textsc{PolarNS} routine: it first scales its input, then applies the cubic Newton--Schulz iteration. Algorithm~\ref{alg:power-iter} supplies the optional spectral-norm estimate. Frobenius scaling uses $\|X\|_{\F}$ directly. Spectral scaling estimates $\|X\|_{\op}$ through power iteration. The estimate need not upper-bound the true norm, so the idealized convergence condition in Theorem~\ref{thm:newton_schulz} does not follow from a finite power-iteration budget alone. The numerical tolerance prevents division by zero.

\subsection{Communication diagrams}

Figure~\ref{fig:schematic-allreduce} shows the int8 path. The collective sums the local signs, and each worker thresholds the result. Figure~\ref{fig:schematic-allgather} shows the packed path, where every worker receives all packed buffers and performs the reduction locally. Both paths keep the momentum and polar computations on each worker. In both diagrams, arrows entering the collective carry signs rather than gradients or polar factors. The lower boxes show the shared update after aggregation; differences in the collective affect buffer storage and network traffic, as summarized in Table~\ref{tab:comm-allgather-vs-allreduce}.

\begin{figure}[!tbp]
\centering
\includegraphics[width=\linewidth]{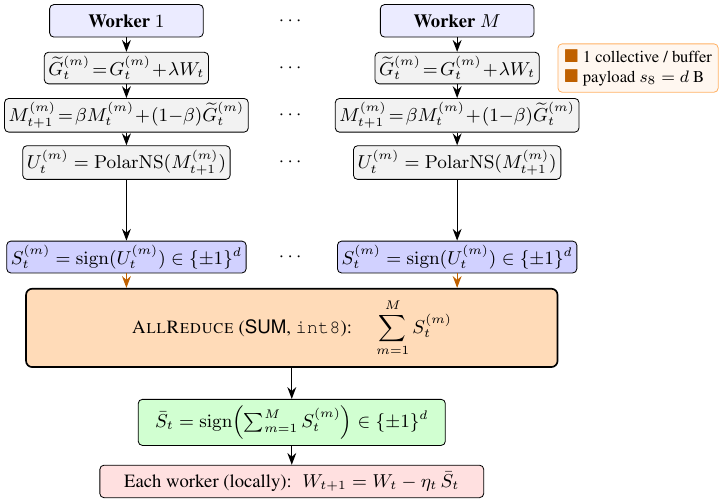}
\caption{Schematic of distributed \textsc{SignMuon} with \textsf{SUM} \textsc{AllReduce} (Algorithm~\ref{alg:sign-muon-dist}). Each worker independently computes its momentum $M_{t+1}^{(m)}$, the Newton--Schulz polar direction $U_t^{(m)}$, and the entrywise sign $S_t^{(m)}\in\{-1,+1\}^d$. A single integer \textsf{SUM} \textsc{AllReduce} aggregates the sign buffers across workers; each worker then thresholds locally with $\sign(\cdot)$ to recover the majority vote $\bar S_t$, and applies the parameter update with no further communication. Normalization and Newton--Schulz are local. For identical input signs and no overflow, this one-buffer collective gives the same vote as packed AllGather; we exclude auxiliary scalar collectives.}
\label{fig:schematic-allreduce}
\end{figure}

\subsection{Pretraining setup and provenance}
\label{app:pretraining-details}
\paragraph{Models and hardware.}
GPT-2 Large, GPT-2 XL, and OPT-350M are trained from random initialization using their native Hugging Face configurations, with all parameters trainable. Table~\ref{tab:5090-models} lists their dimensions and parameter counts. OPT-350M retains its 512-dimensional word-embedding projection and native positional capacity of 2,048; all experiments use a context length of 1,024.

\begin{table}[H]\centering\small
\caption{Native model configurations used in pretraining, with 331 million to 1.558 billion trainable parameters.}
\label{tab:5090-models}
\begin{tabular}{lrrrr}\toprule
Model & Layers & Attention heads & Hidden width & Trainable parameters\\\midrule
GPT-2 XL & 48 & 25 & 1600 & 1,557,611,200\\
GPT-2 Large & 36 & 20 & 1280 & 774,030,080\\
OPT-350M & 24 & 16 & 1024 & 331,196,416\\\bottomrule
\end{tabular}
\end{table}

Single-seed experiments, including the memory measurements in Figure~\ref{fig:memory-xl} and Table~\ref{tab:memory-models}, use eight NVIDIA GeForce RTX 5090 GPUs with approximately 32 GiB per device. For the multi-seed experiments, GPT-2 XL uses eight RTX 5090 GPUs, while GPT-2 Large and OPT-350M use eight NVIDIA Quadro RTX 6000 GPUs.

The software environment comprises PyTorch 2.10.0+cu128, Transformers 5.13.0, Datasets 5.0.0, Tokenizers 0.22.2, and NumPy 2.4.1, with NVIDIA driver 595.71.05. The PyTorch build uses CUDA 12.8.
\paragraph{Data and evaluation.}
We evaluate the optimizers on subsets of three language-modeling datasets: TinyStories, FineMath-4+, and FineWeb-Edu. FineWeb-Edu uses the \texttt{sample-10BT} configuration. Each subset contains approximately 100,000 documents after removing exact duplicates, duplicates identified after text normalization, and near-duplicates. Near-duplicate detection compares sets of five consecutive lexical tokens and removes documents with Jaccard similarity of at least 85\% to a retained document. A random partition with seed 42 assigns 90\% of the retained documents to training and 10\% to validation. Before model tokenization, validation documents sharing an identical passage of at least 50 consecutive lexical tokens with a training document are removed. Lexical tokens are defined here as words and punctuation marks.

The remaining documents are encoded using each model's native tokenizer, terminated with an end-of-sequence token, and concatenated separately within each split. A second overlap check operates on the resulting token streams. Candidate validation windows are excluded if they contain an exact match of at least 64 consecutive model tokens with the training stream. Unlike lexical tokens, these model tokens can represent word fragments. From the remaining candidates, 1,000 fixed, nonoverlapping windows are selected for evaluation. Each validation window may span multiple documents and provides 1,024 input tokens with 1,024 corresponding next-token prediction targets. All optimizers use the same training and validation partitions and fixed validation sequences for each dataset and model. Evaluation is performed at initialization and every 50 optimizer updates, with each evaluation covering 1,024,000 next-token prediction targets.

\paragraph{Metrics and aggregation.}
For $N$ evaluated next-token targets $x_i$ and their corresponding contexts $c_i$, validation loss and perplexity are defined as
\[
\ell_{\mathrm{val}}(W)
=-\frac{1}{N}\sum_{i=1}^{N}\log p_W(x_i\mid c_i),
\qquad
\mathrm{PPL}(W)=\exp\!\bigl(\ell_{\mathrm{val}}(W)\bigr).
\]
Natural logarithms are used, so validation loss equals log-perplexity and is expressed in nats per target token. Figure~\ref{fig:pretraining-single-tokens} plots single-seed validation loss against cumulative training tokens processed. For validation losses $y_1,y_2,y_3$ from three training seeds, the mean and sample variance are
\[
\bar y=\frac{1}{3}\sum_{r=1}^{3}y_r,
\qquad
s^2=\frac{1}{2}\sum_{r=1}^{3}(y_r-\bar y)^2.
\]
Table~\ref{tab:pretraining-combined} reports these statistics using each seed's last recorded validation loss within the 1,800-second comparison window. Figure~\ref{fig:pretraining-multi-time} plots the mean validation loss against elapsed time, with bands indicating one sample standard deviation above and below the mean. For each optimizer, seed trajectories are linearly interpolated at common time points within the interval covered by all three seeds. The bands represent variation across training seeds.

\paragraph{Training and optimizer settings.}
All four methods use eight GPU workers, each holding a complete model replica, with a global batch size of 64 sequences. Each worker processes one sequence per microbatch and accumulates gradients over eight microbatches, giving 65,536 training tokens per optimizer update. Gradients are accumulated locally on each worker and clipped to a total Euclidean norm of at most 1 before the optimizer step. Learning rates remain constant, with no warmup or decay; dropout and weight decay are zero. Parameters, gradients, and optimizer states are stored in FP32. Computation uses BF16 autocast where supported and FP16 autocast with synchronized dynamic loss scaling otherwise. The RTX 5090 runs use BF16 autocast with TF32 enabled. Training uses non-reentrant gradient checkpointing and scaled dot-product attention, with vocabulary projection and loss computation performed in chunks of 256 token positions.

SignMuon uses momentum 0.9 and 10 Newton--Schulz iterations with Frobenius normalization and $\varepsilon=10^{-8}$. The matrix transform is applied to all trainable parameter matrices, including embedding matrices. Parameter vectors, such as biases and layer-normalization weights, use signed momentum. SignAdam uses bias-corrected first and second moments with coefficients $(0.9,0.999)$ and $\varepsilon=10^{-8}$. SignLion uses coefficients 0.9 for the update direction and 0.99 for the momentum update. SignSGD uses no momentum.

\paragraph{Communication.}
All four methods form directions from local gradients and aggregate their signs through packed AllGather and majority vote. Each worker packs eight signs per byte into uint8 buffers, with at most 1,048,576 parameter elements per communication bucket. Unpacked signs are accumulated in an int16 vote buffer. Zero local directions and tied votes map to $+1$. 
\paragraph{Learning-rate selection.}
For each architecture and optimizer, we searched the learning-rate grid $\{10^{-3},10^{-4},10^{-5},10^{-6}\}$. Table~\ref{tab:5090-rates} lists the selected rates, which remain fixed across the three datasets. 

\begin{table}[H]\centering\small
\caption{Selected pretraining learning rates, fixed across the three datasets.}
\label{tab:5090-rates}
\begin{tabular}{lcccc}\toprule
Model & SignMuon & SignAdam & SignLion & SignSGD\\\midrule
GPT-2 XL & $10^{-4}$ & $10^{-5}$ & $10^{-5}$ & $10^{-4}$\\
GPT-2 Large & $10^{-4}$ & $10^{-4}$ & $10^{-4}$ & $10^{-4}$\\
OPT-350M & $10^{-5}$ & $10^{-5}$ & $10^{-5}$ & $10^{-5}$\\\bottomrule
\end{tabular}
\end{table}

\paragraph{Budgets and seeds.}
Comparisons use validation measurements recorded within the first 1,800 seconds of elapsed time, including training, validation, and communication but excluding model and data setup. The number of completed updates varies by optimizer and seed. The single-seed experiments use a fixed random seed of 1337. For GPT-2 XL, these runs are combined with runs using two additional randomly selected seeds. GPT-2 Large and OPT-350M share a set of three randomly selected seeds. Within each architecture, the same seeds are used across all datasets and optimizers, with their values recorded in the experiment configurations in the code repository. The seed for data sampling, document splitting, and validation selection is fixed at 42. Runtime comparisons are restricted to experiments using the same architecture and GPU configuration.

\subsection{Additional nanoGPT comparisons}
\label{app:nanogpt-scaling-details}

We compare Frobenius and spectral scaling in SignMuon and evaluate distributed weak scaling using nanoGPT (Figures~\ref{fig:nanogpt_scaling_compare} and~\ref{fig:weak-scaling-nanogpt}). The model consists of six Transformer layers with six attention heads per layer, embedding dimension 384, context length 512, and approximately 29.94 million parameters excluding positional embeddings. Text is encoded using the GPT-2 tokenizer.

The normalization comparison evaluates five combinations of learning rate and Newton--Schulz iteration count, pairing Frobenius and spectral scaling at each setting. Both methods normalize the momentum matrix before the Newton--Schulz iterations: Frobenius scaling uses its Frobenius norm, while spectral scaling uses an estimate of its largest singular value. The training iteration limit is 1,500, with a linear learning-rate warmup spanning 2,000 iterations. 

The weak-scaling experiments use 1, 4, 8, and 16 GPUs while holding the workload per GPU fixed. Each GPU processes 30 sequences of 512 tokens per microbatch with two gradient accumulation steps, giving 30,720 tokens per GPU per update. SignMuon, SignAdam, and SignSGD share a target learning rate of $10^{-3}$, linear warmup, and zero weight decay. SignMuon uses ten Newton--Schulz iterations with spectral scaling.
\subsection{Memory measurement details}
\label{app:memory-details}
Figure~\ref{fig:memory-xl} reports GPT-2 XL memory measurements on TinyStories with seed 1337. For each method we take the largest recorded
training-allocation value and persistent-state value. These are per-GPU values,
not a sum over eight GPUs. The trainer resets allocation peaks before updates,
separates training from validation, and reduces memory statistics with a maximum
across workers. Persistent state counts optimizer tensors only; training peaks
also include parameters, gradients, activations, and temporary workspaces.

Table~\ref{tab:memory-models} extends Figure~\ref{fig:memory-xl} to OPT-350M and GPT-2 Large. Across all three models, SignMuon and SignLion have the same persistent-state footprint, approximately half that of SignAdam, while SignSGD stores no optimizer state. Peak training allocation includes optimizer state, model parameters, gradients, activations, and temporary tensor allocations. We therefore report optimizer-state size separately from peak training allocation.

\begin{table}[H]\centering\small
\caption{TinyStories memory measurements (seed 1337; GiB per GPU). SignMuon and SignLion use the same optimizer-state memory, half that of SignAdam; SignSGD has no persistent state.}
\label{tab:memory-models}
\begin{tabular}{lrrrrrr}\toprule
 & \multicolumn{2}{c}{OPT-350M} & \multicolumn{2}{c}{GPT-2 Large} & \multicolumn{2}{c}{GPT-2 XL}\\
\cmidrule(lr){2-3}\cmidrule(lr){4-5}\cmidrule(lr){6-7}
Method & Peak & State & Peak & State & Peak & State\\\midrule

SignMuon & 4.528 & 1.234 & 10.526 & 2.883 & 21.027 & 5.803 \\
SignAdam & 5.762 & 2.468 & 13.504 & 5.767 & 26.963 & 11.605 \\
SignLion & 4.528 & 1.234 & 10.540 & 2.883 & 21.034 & 5.803 \\
SignSGD & 3.294 & 0.000 & 7.606 & 0.000 & 15.124 & 0.000 \\
\bottomrule

\end{tabular}
\end{table}
For $P$ FP32 parameters, the stored tensors occupy $4P$ bytes for SignMuon and
SignLion, $8P$ bytes for SignAdam, and zero for SignSGD. In GPT-2 XL this gives
approximately 5.80, 11.61, and 0 GiB, respectively. Peak allocation is workload and
implementation dependent; the persistent-state formula alone does not predict it.

We report allocated GPU tensor memory in GiB ($2^{30}$ bytes). CUDA reserved memory and external
device-level telemetry are different metrics and are not substituted for it.

\section{Five-minute multi-seed nanoGPT experimental setup}
\label{sec:nanogpt-setup}
We evaluate learning-rate sensitivity and optimization performance on TinyStories, FineMath-4+, and FineWeb-Edu under a five-minute wall-clock budget per run.
\subsection{Model, hardware, and experimental design}
The experiments comprise 144 runs across three datasets, four signed optimizers,
learning rates $\{10^{-6},10^{-5},10^{-4},10^{-3}\}$, and training seeds
$\{42,123,2026\}$. Each run uses a single NVIDIA Quadro RTX 6000 GPU.

The randomly initialized GPT uses six Transformer layers, six attention heads,
embedding dimension 384, context length 512, zero dropout, and no linear or layer-normalization
biases. Token embeddings and the language-model head share weights. The GPT-2
tokenizer has 50,257 tokens, and the model vocabulary is padded to 50,304.
The model has approximately 30.14 million parameters (29.94 million excluding
positional embeddings).

\begin{table}[H]\centering\small
\caption{Training and optimizer settings for the five-minute nanoGPT experiments.}
\label{tab:nano-settings}
\begin{tabularx}{\linewidth}{@{}lX@{}}\toprule
Setting & Value\\\midrule
Batching & 12 sequences per microbatch, five accumulation steps\\
Tokens per update & $12\times5\times512=30{,}720$ on one GPU\\
Precision & FP16 autocast with dynamic loss scaling; FP32 model parameters\\
Compilation & Disabled\\
Schedule & Constant learning rate; no warmup; weight decay zero\\
Gradient clipping & Norm threshold 1.0, after unscaling\\
SignMuon & $\beta=0.9$; 10 Newton--Schulz iterations in FP32; Frobenius normalization; $\epsilon=10^{-12}$\\
SignAdam & $(\beta_1,\beta_2)=(0.9,0.999)$; $\epsilon=10^{-8}$\\
SignLion & $(\beta_1,\beta_2)=(0.9,0.95)$\\
SignSGD & No momentum and no persistent optimizer tensor\\
Logging / validation & Every 10 / 50 updates; final validation at the stopping step\\
Evaluation sampling & 20 batches of 12 sequences per split; fixed sampling generators\\\bottomrule
\end{tabularx}
\end{table}

Table~\ref{tab:nano-settings} summarizes the training and optimizer settings.

The sign operators use $\sgn(0)=0$. Python, NumPy, and CUDA random generators use
the specified training seed. TF32 is enabled and cuDNN benchmarking is disabled.
Deterministic algorithms are requested where supported, but attention kernels
can still introduce nondeterminism.

\subsection{Datasets and sampling}
Dataset preparation uses fixed revisions and removes exact duplicates after
case folding and whitespace normalization. FineMath-4+ and FineWeb-Edu are
sampled through a 10,000-document shuffle buffer, retaining at most 100,000
unique nonempty documents each. The TinyStories subset contains 104,919
documents after deduplication. A random document split with seed 42 assigns
approximately 10\% of each subset to validation. An end-of-text token is
appended to each encoded document, and tokens are stored as uint16 memory-mapped
streams separately for training and validation.

\begin{table}[H]\centering\small
\caption{Dataset sizes after deduplication and document-level splitting for the nanoGPT experiments.}
\label{tab:nano-corpora}
\begin{tabular}{lrrrr}\toprule
Dataset & Train docs & Validation docs & Train tokens & Validation tokens\\\midrule
TinyStories & 94,346 & 10,573 & 20,741,427 & 2,310,099\\
FineMath-4+ & 89,951 & 10,049 & 138,102,256 & 15,826,128\\
FineWeb-Edu & 89,951 & 10,049 & 93,789,022 & 10,257,169\\\bottomrule
\end{tabular}
\end{table}
Table~\ref{tab:nano-corpora} reports the resulting corpus sizes.
Training samples random contiguous token windows. Evaluation resets a dedicated
CPU sampling generator on every call: seed 42 for training-loss estimation and
43 for validation. It therefore uses matched validation windows across methods,
learning rates, and training seeds without advancing the training sampler.
Each validation uses $20\times12\times512=122{,}880$ target tokens.

\subsection{Timing and result selection}
The 300-second budget begins after initial evaluation and includes periodic
validation and logging. Training stops at the first update boundary after the
budget is reached, so measured durations can slightly exceed 300 seconds.
Setup, dataset preparation, and final validation are excluded from the budget.

For each dataset and method, the learning rate is selected by the lowest mean
final validation loss across the three seeds. We report the mean and sample
standard deviation over those same seeds. Appendix~\ref{sec:lr}
reports results for all four learning rates, and Appendix~\ref{sec:nanogpt-results}
describes the construction of the learning curves.

\section{Implemented signed optimizer algorithms}
\label{app:executed-algorithms}
The following algorithms specify the zero-weight-decay setting used in both
campaigns. At update $t$, worker $r$ first accumulates its local microbatch
gradients, unscales them when using FP16, and clips the resulting gradient to norm
1. Denote a clipped parameter-block gradient by $G_t^{(r)}$ and the shared
parameter by $W_t$. All workers initialize at the same $W_0$; persistent momentum and second-moment buffers start at zero and match their parameters. We use zero-based updates, as in Algorithm~\ref{alg:sign-muon}. Powers, square roots, and divisions in the Adam recurrence act entrywise; $A^{\odot2}$ denotes entrywise squaring.

\subsection{Local signs and distributed majority vote}
Define $\operatorname{Vote}(D_t^{(1)},\ldots,D_t^{(M)})$ according to the executed campaign.
For single-worker nanoGPT, $\operatorname{Vote}(D)=\sgn(D)$, including $\sgn(0)=0$.
For eight-worker distributed pretraining, form
\[
 S_t^{(r)}=2\,\mathbf{1}\{D_t^{(r)}\geq0\}-1,\qquad
 \overline S_t=2\,\mathbf{1}\!\left\{\sum_{r=1}^{M}S_t^{(r)}\geq0\right\}-1.
\]
Each worker packs eight local sign entries per byte, AllGathers the packed
buffers, and accumulates unpacked signs into an int16 vote buffer. All workers
then apply the same $\overline S_t$. The optimizer state remains local; gradients
are not averaged before constructing directions. This shared vote operation is
used by all four algorithms below.

\subsection{SignAdam, SignLion, and SignSGD}
Algorithms~\ref{alg:adam}, \ref{alg:lion}, and~\ref{alg:sgd} use the same clipped local gradients and vote operation, differing in how they form the direction and what state they retain. SignAdam signs its bias-corrected moment ratio; SignLion forms its direction before updating the stored momentum; SignSGD directly signs the current gradient.
\begin{algorithm}[H]
\caption{SignAdam: signed bias-corrected Adam direction}
\label{alg:adam}
\begin{algorithmic}[1]
\REQUIRE $\eta$; $(\beta_1,\beta_2)=(0.9,0.999)$; $\epsilon=10^{-8}$
\STATE Initialize $M_0^{(r)}\gets0$, $V_0^{(r)}\gets0$.
\FOR{$t=0,1,\ldots$}
  \STATE Compute the accumulated, unscaled, clipped local gradient $G_t^{(r)}$.
  \STATE $M_{t+1}^{(r)}\gets\beta_1M_t^{(r)}+(1-\beta_1)G_t^{(r)}$.
  \STATE $V_{t+1}^{(r)}\gets\beta_2V_t^{(r)}+(1-\beta_2)(G_t^{(r)})^{\odot2}$.
  \STATE $\widehat M_{t+1}^{(r)}\gets M_{t+1}^{(r)}/(1-\beta_1^{t+1})$;
         $\widehat V_{t+1}^{(r)}\gets V_{t+1}^{(r)}/(1-\beta_2^{t+1})$.
  \STATE $D_t^{(r)}\gets\widehat M_{t+1}^{(r)}/(\sqrt{\widehat V_{t+1}^{(r)}}+\epsilon)$, entrywise.
  \STATE $\overline S_t\gets\operatorname{Vote}(D_t^{(1)},\ldots,D_t^{(M)})$.
  \STATE $W_{t+1}\gets W_t-\eta\overline S_t$.
\ENDFOR
\end{algorithmic}
\end{algorithm}
SignAdam retains both first and second moments in the executed implementation.
For finite entries, its positive denominator and bias-correction factors do not
change the sign of the first moment. This algebraic observation does not remove
the second-moment tensor from the measured implementation.

\begin{algorithm}[H]
\caption{SignLion: interpolate the old momentum, then update its stored EMA}
\label{alg:lion}
\begin{algorithmic}[1]
\REQUIRE $\eta$; $\beta_1=0.9$; $\beta_2=0.99$ for distributed pretraining or $0.95$ for nanoGPT 
\STATE Initialize $M_0^{(r)}\gets0$.
\FOR{$t=0,1,\ldots$}
  \STATE Compute the accumulated, unscaled, clipped local gradient $G_t^{(r)}$.
  \STATE $D_t^{(r)}\gets\beta_1M_t^{(r)}+(1-\beta_1)G_t^{(r)}$.
  \STATE $M_{t+1}^{(r)}\gets\beta_2M_t^{(r)}+(1-\beta_2)G_t^{(r)}$.
  \STATE $\overline S_t\gets\operatorname{Vote}(D_t^{(1)},\ldots,D_t^{(M)})$.
  \STATE $W_{t+1}\gets W_t-\eta\overline S_t$.
\ENDFOR
\end{algorithmic}
\end{algorithm}
Both beta coefficients act on one momentum buffer. The direction uses the old
buffer before the $\beta_2$ update, so SignLion has one persistent tensor per parameter.

\begin{algorithm}[H]
\caption{SignSGD: signed local gradient with no momentum}
\label{alg:sgd}
\begin{algorithmic}[1]
\REQUIRE Learning rate $\eta$; no persistent optimizer tensors
\FOR{$t=0,1,\ldots$}
  \STATE Compute the accumulated, unscaled, clipped local gradient $G_t^{(r)}$.
  \STATE $\overline S_t\gets\operatorname{Vote}(G_t^{(1)},\ldots,G_t^{(M)})$.
  \STATE $W_{t+1}\gets W_t-\eta\overline S_t$.
\ENDFOR
\end{algorithmic}
\end{algorithm}
These algorithms describe the implementations used in the pretraining
experiments. The separate weak-scaling configuration in
Appendix~\ref{app:nanogpt-scaling-details} uses the normalization specified for that experiment.

\section{Additional nanoGPT pretraining results}
\label{sec:nanogpt-results}
The experiments follow the five-minute training and evaluation protocol in
Appendix~\ref{sec:nanogpt-setup}.

\subsection{Performance under a five-minute budget}
Table~\ref{tab:final} summarizes final validation performance at the learning
rate selected for each dataset and optimizer. Perplexity is computed for each
seed before averaging.

\begin{table}[H]\centering\small
\caption{Final nanoGPT results at the selected learning rate $\eta$ under a
five-minute wall-clock budget. Loss (nats) and perplexity show mean $\pm$ sample
standard deviation over three training seeds. Training tokens (millions) and
measured run durations (seconds) are means; durations exclude final validation.}
\label{tab:final}
\setlength{\tabcolsep}{3pt}
\footnotesize
\begin{tabular}{lllrrrr}\toprule
Dataset & Method & $\eta$ & Validation loss & Perplexity & Tokens (M) & Time (s)\\\midrule

TinyStories & SignMuon & $10^{-3}$ & $1.999\pm0.004$ & $7.38\pm0.03$ & 29.20 & 300.29 \\
TinyStories & SignAdam & $10^{-4}$ & $2.204\pm0.005$ & $9.06\pm0.05$ & 36.99 & 300.16 \\
TinyStories & SignLion & $10^{-4}$ & $2.132\pm0.014$ & $8.43\pm0.12$ & 36.99 & 300.16 \\
TinyStories & SignSGD & $10^{-4}$ & $3.425\pm0.016$ & $30.74\pm0.51$ & 37.42 & 300.12 \\
\addlinespace
FineMath-4+ & SignMuon & $10^{-3}$ & $3.470\pm0.012$ & $32.13\pm0.40$ & 29.19 & 300.11 \\
FineMath-4+ & SignAdam & $10^{-4}$ & $3.898\pm0.007$ & $49.33\pm0.33$ & 36.88 & 300.49 \\
FineMath-4+ & SignLion & $10^{-4}$ & $3.658\pm0.023$ & $38.81\pm0.88$ & 37.02 & 300.17 \\
FineMath-4+ & SignSGD & $10^{-4}$ & $5.393\pm0.007$ & $219.97\pm1.64$ & 37.44 & 300.13 \\
\addlinespace
FineWeb-Edu & SignMuon & $10^{-3}$ & $4.696\pm0.012$ & $109.52\pm1.26$ & 29.03 & 300.20 \\
FineWeb-Edu & SignAdam & $10^{-4}$ & $5.193\pm0.014$ & $180.08\pm2.51$ & 36.86 & 300.27 \\
FineWeb-Edu & SignLion & $10^{-4}$ & $5.050\pm0.018$ & $156.06\pm2.87$ & 36.92 & 300.11 \\
FineWeb-Edu & SignSGD & $10^{-4}$ & $6.450\pm0.003$ & $632.91\pm1.70$ & 37.42 & 300.14 \\
\addlinespace
\bottomrule

\end{tabular}
\end{table}
SignMuon achieves the lowest mean final validation loss on each dataset, processing about
29 million tokens versus approximately 37 million for the other methods.

\subsection{Token efficiency and estimated wall-clock time}
\label{sec:nanogpt-curves}
Figure~\ref{fig:nano-curves} compares validation loss against cumulative
training tokens and estimated elapsed time. After $k$ updates, each run has
processed $N(k)=30{,}720k$ training tokens. At each selected learning rate,
token curves average the losses at evaluation checkpoints shared by all three
seeds, ending at their last common 50-update evaluation. Final results at each
run's stopping point are reported separately in Table~\ref{tab:final}.

Elapsed time is measured at the end of each run but is not recorded at every
validation checkpoint. The timing accumulator sums durations from updates logged
every ten iterations and cannot recover intermediate elapsed times. For seed
$r$, let $T_r$ denote the measured final duration and $K_r$ the number of completed
updates. The estimated elapsed time at update $k$ is
\[
 \widehat t_r(k)=T_r\frac{k}{K_r}.
\]
This approximation assumes a constant average update rate and distributes
validation and logging overhead uniformly over the trajectory. Seed trajectories
are linearly interpolated on a five-second grid between 0 and 300 seconds, within
their observed estimated-time ranges and without extrapolation. The resulting
curves summarize mean progress on this estimated time axis. Bands represent
variation across training seeds and exclude uncertainty from time estimation
and learning-rate selection.

\begin{figure}[!tbp]\centering
\resizebox{0.90\linewidth}{!}{%
\begin{minipage}{\linewidth}\centering

\begin{minipage}{0.49\linewidth}\centering\includegraphics[width=\linewidth]{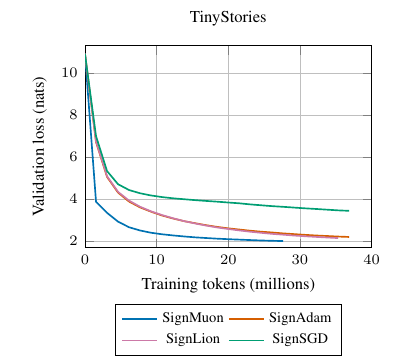}\end{minipage}\hfill
\begin{minipage}{0.49\linewidth}\centering\includegraphics[width=\linewidth]{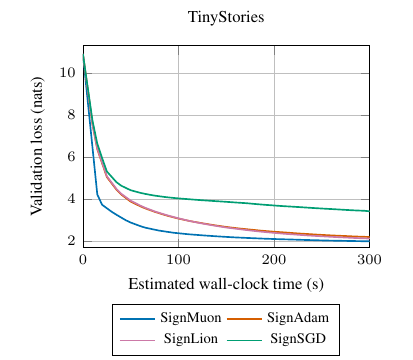}\end{minipage}
\vspace{0.4cm}
\begin{minipage}{0.49\linewidth}\centering\includegraphics[width=\linewidth]{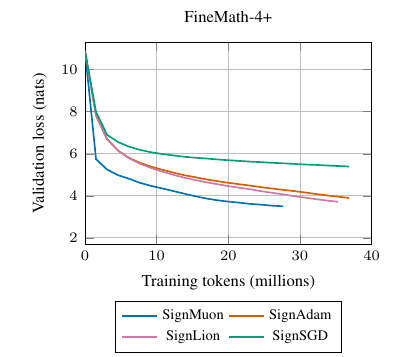}\end{minipage}\hfill
\begin{minipage}{0.49\linewidth}\centering\includegraphics[width=\linewidth]{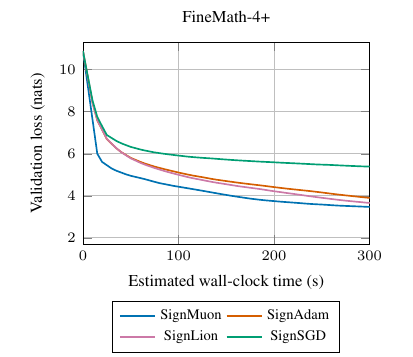}\end{minipage}
\vspace{0.4cm}
\begin{minipage}{0.49\linewidth}\centering\includegraphics[width=\linewidth]{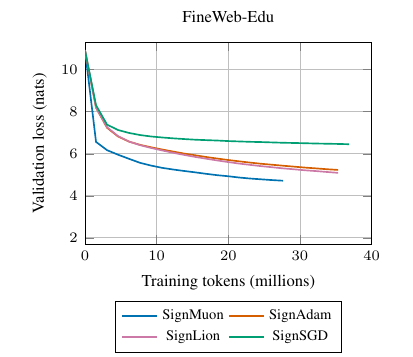}\end{minipage}\hfill
\begin{minipage}{0.49\linewidth}\centering\includegraphics[width=\linewidth]{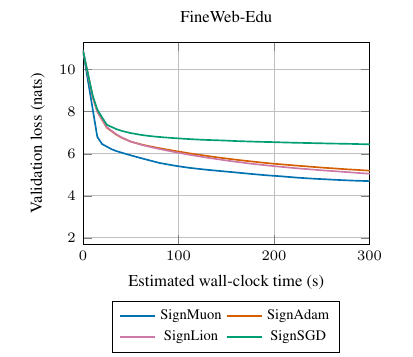}\end{minipage}
\end{minipage}%
}\par
\caption{nanoGPT validation loss at the selected learning rates against training
tokens (left) and estimated wall-clock time (right). Time coordinates use each
run's measured final duration and completed-update count. Curves show means and
shading indicates one sample standard deviation over three training seeds.
SignMuon has the lowest mean loss at the shared token and estimated-time
endpoints. Intermediate time coordinates are estimates, not measured validation
timestamps.}
\label{fig:nano-curves}
\end{figure}

\subsection{Learning-rate sensitivity}
\label{sec:lr}
Figure~\ref{fig:lr} compares final validation loss across all four tested
learning rates using the full set of 144 runs. The lowest mean loss occurs at
$10^{-3}$ for SignMuon and $10^{-4}$ for the other methods on all three datasets.
Smaller rates yield higher final losses within the five-minute budget, while
increasing the rate to $10^{-3}$ also degrades the other methods' final losses.

\begin{figure}[!tbp]\centering
\begin{minipage}{0.49\linewidth}\centering\includegraphics[width=\linewidth]{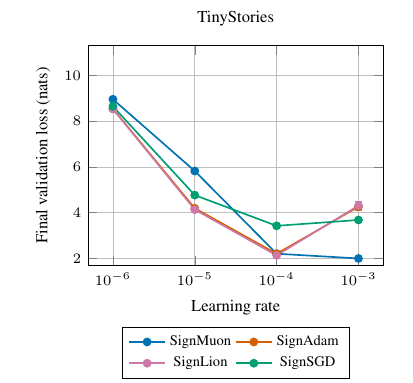}\end{minipage}\hfill
\begin{minipage}{0.49\linewidth}\centering\includegraphics[width=\linewidth]{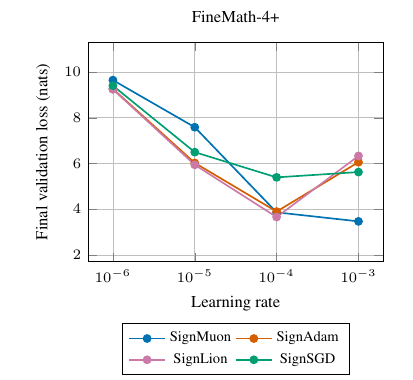}\end{minipage}
\vspace{0.6cm}
\begin{minipage}{0.49\linewidth}\centering\includegraphics[width=\linewidth]{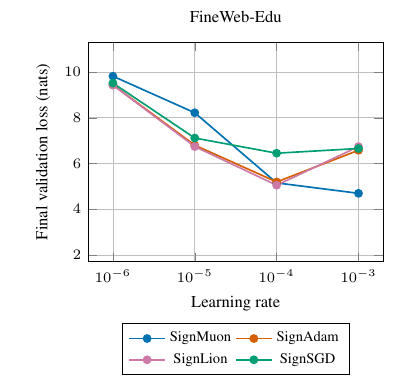}\end{minipage}
\caption{Final nanoGPT validation loss (nats) versus learning rate on a
logarithmic horizontal axis. Each point is the mean over three training seeds,
with error bars indicating one sample standard deviation. Lines connect the
four tested learning rates.}
\label{fig:lr}
\end{figure}

\section{Communication complexity}
\label{app:comm}

We count the messages and bytes needed to aggregate signs. The model excludes gradient evaluation, local polar computation, packing, and unpacking. Let $M$ denote the worker count and let $d=\sum_{\ell=1}^L m_\ell n_\ell$ denote the total number of communicated entries. Concatenation reduces collective calls. The formulas below describe one buffer; for bucketed training, sum the per-bucket costs and byte padding. We exclude scalar checks, logging, and other auxiliary collectives.

\subsection{Majority vote and local operations}

At iteration $t$, worker $p$ computes a direction $U_{\ell,t}^{(p)}$ for block $\ell$ and its binary sign matrix $S_{\ell,t}^{(p)}$. Workers aggregate these signs as
\begin{equation}
\label{eq:maj-vote}
\bar S_{\ell,t}=\sign\!\left(\sum_{p=1}^M S_{\ell,t}^{(p)}\right).
\end{equation}
The binary implementation maps a zero sum to $+1$. A different fixed tie rule changes the update at ties but does not change the message size.

\begin{lemma}[Majority vote through integer reduction]
\label{lem:maj-reduction}
Workers can compute \eqref{eq:maj-vote} by encoding the signs as integers, summing them elementwise, and applying a local sign operation. The accumulator must represent every integer in $[-M,M]$.
\end{lemma}
\begin{proof}
Each coordinate contributes either $-1$ or $+1$, so its sum lies in $\{-M,-M+2,\ldots,M\}$. Exact integer addition gives this sum, and thresholding it gives the majority vote with the chosen tie rule.
\end{proof}

Once every worker has $\bar S_{\ell,t}$, it can normalize the matrix or compute its polar factor locally. For example,
\[
D_{\ell,t}=\frac{\bar S_{\ell,t}}{\sqrt{m_\ell n_\ell}}
\quad\text{or}\quad D_{\ell,t}=\polar(\bar S_{\ell,t})
\]
requires no further messages before the update $W_{\ell,t+1}=W_{\ell,t}-\eta_tD_{\ell,t}$. The two directions have different optimization properties despite sharing the same communication cost.

\subsection{Latency--bandwidth model}

The cost model separates message startup from byte transfer. This lets us compare the collective paths in Figures~\ref{fig:schematic-allreduce} and~\ref{fig:schematic-allgather} using the same network parameters, and explains why fewer bytes need not give a proportional reduction in total time. Table~\ref{tab:comm-allgather-vs-allreduce} summarizes the resulting per-worker traffic and round counts.

\begin{definition}[Message cost]
A point-to-point message of $s$ bytes takes modeled time
\[
T_{\alpha,\beta}(s)=\alpha+\beta s,
\]
where $\alpha>0$ is startup latency and $\beta>0$ is inverse bandwidth in seconds per byte. In this appendix, $\beta$ denotes the network parameter, not momentum.
\end{definition}

We assume simultaneous sending and receiving, no network contention, balanced chunks, and negligible local reduction time. The bandwidth terms count one direction of traffic; sent-plus-received volume is a separate accounting quantity. We ignore padding needed to balance chunks. Within this idealized model, we count a collective's critical-path rounds and per-worker traffic \citep{rabenseifner2004allreduce,thakur2005optimization,hoefler2010characterizing}. A ring AllReduce performs reduce-scatter followed by AllGather, using $2(M-1)$ rounds. Each phase sends and receives $(1-1/M)s$ bytes per worker. Recursive halving/doubling uses $2\log_2M$ rounds for power-of-two $M$ with the same bandwidth term.

For a pipelined tree, we use the approximations
\begin{align}
T_{\mathrm{bcast}}(s,M)&\approx\alpha\lceil\log_2M\rceil+\beta s,\\
T_{\mathrm{reduce}}(s,M)&\approx\alpha\lceil\log_2M\rceil+\beta s.
\end{align}
Without pipelining, sending the complete buffer in every round instead gives a bandwidth term proportional to $\beta s\lceil\log_2M\rceil$. The tree results below use this idealized pipelined model, omitting additional segment-startup and pipeline-drain costs; they are not exact latency guarantees for a particular backend.

\subsection{Integer AllReduce}
\label{app:comm-allreduce}

For a reduction datatype with $b$ bits per entry, the buffer size is
\begin{equation}
\label{eq:payload-s}
s=bd/8\quad\text{bytes}.
\end{equation}
Our int8 implementation uses $b=8$, so $s=s_8=d$. A packed one-bit representation cannot directly carry the integer partial sums of an ordinary \textsf{SUM} AllReduce. We analyze packed signs through AllGather in Section~\ref{app:comm-allgather}.

\begin{theorem}[AllReduce communication cost]
\label{thm:allreduce}
Suppose workers concatenate their integer sign tensors into one $s$-byte buffer and use a datatype that avoids overflow. One AllReduce followed by local thresholding has modeled cost
\begin{equation}
\label{eq:allreduce-general}
T_{\mathrm{AR}}\approx\alpha R(M)+2(1-1/M)s\beta,
\end{equation}
where $R(M)=2(M-1)$ for a ring and $R(M)=2\log_2M$ for recursive halving/doubling with power-of-two $M$. Each worker sends $2(1-1/M)s$ bytes and receives the same amount. The sum of sent and received bytes is $4(1-1/M)s$.
\end{theorem}
\begin{proof}
Lemma~\ref{lem:maj-reduction} reduces the computation to an integer sum. Reduce-scatter and AllGather each send $(1-1/M)s$ bytes per worker. Adding the two phases gives the bandwidth term; counting their rounds gives the latency term. Receives match sends, so bidirectional traffic is twice the sent volume.
\end{proof}

If workers instead call AllReduce separately for each of $L$ blocks, the payload sum stays $s$ but the modeled latency becomes $\alpha LR(M)$. Buffer concatenation therefore matters when startup cost is substantial. This observation concerns call organization, not additional quantization.

For int8, Theorem~\ref{thm:allreduce} specializes to
\begin{equation}
\label{eq:int8-size}
s_8=d,
\end{equation}
\begin{equation}
\label{eq:allreduce-time}
T_{\mathrm{AR(int8)}}\approx\alpha R(M)+2(1-1/M)d\beta.
\end{equation}
The accumulator range limits the use of a fixed-width integer type. With int8 signs and int8 accumulation, $M\le127$ avoids overflow for every possible sign sum. The experiments use at most 16 workers.

\subsection{Packed AllGather}
\label{app:comm-allgather-vs-allreduce}
\label{app:comm-allgather}

Worker $p$ represents a positive sign by $b_i^{(p)}=1$ and a negative sign by $b_i^{(p)}=0$. Packing these bits into a byte array $B_t^{(p)}$ gives
\begin{equation}
\label{eq:packed-size}
s_1=\lceil d/8\rceil\quad\text{bytes}.
\end{equation}
Each byte belongs to $\{0,\ldots,255\}$; unused bits in the final byte do not represent parameters. After AllGather, every worker has all $M$ buffers. It counts $c_i=\sum_p b_i^{(p)}$ and computes
\[
\sum_p S_t^{(p)}[i]=2c_i-M,\qquad \bar S_t[i]=\sign(2c_i-M).
\]
The vote requires local computation only.

\paragraph{Traffic and time.}
Every worker must receive the other workers' $(M-1)s_1$ bytes. A bandwidth-optimal AllGather has
\begin{align}
\label{eq:allgather-volume}
V_{\mathrm{send}}^{\mathrm{AG}}&=(M-1)s_1,\\
V_{\mathrm{recv}}^{\mathrm{AG}}&=(M-1)s_1,\\
V_{\mathrm{bidir}}^{\mathrm{AG}}&=2(M-1)s_1.
\end{align}
Summing over workers gives $M(M-1)s_1$ bytes sent across the system and the same volume received. The modeled time is
\begin{equation}
\label{eq:allgather-time}
T_{\mathrm{AG(1bit)}}\approx\alpha R_{\mathrm{AG}}(M)+(M-1)s_1\beta,
\end{equation}
where a ring uses $R_{\mathrm{AG}}(M)=M-1$, and recursive doubling uses $\log_2M$ rounds for power-of-two $M$.

\paragraph{Local cost.}
AllGather requires storage for $Ms_1$ bytes and a vote computation on each worker. Unpacking and summing every sign uses $O(Md)$ integer operations. Byte-oriented or bitwise implementations can process the packed representation, but they still read $O(Ms_1)$ bytes. These operations do not change network traffic, yet their cost grows with $M$.

\subsection{AllGather versus AllReduce}
\label{app:comm-comparison}

Table~\ref{tab:comm-allgather-vs-allreduce} distinguishes the local payload from total traffic. For $M>1$, ignoring byte padding, the ratio of the bandwidth terms is
\begin{equation}
\label{eq:ratio-bandwidth}
\frac{(M-1)s_1}{2(1-1/M)s_8}
\approx\frac{(M-1)d/8}{2(1-1/M)d}=\frac{M}{16}.
\end{equation}
Packed AllGather therefore has a smaller modeled bandwidth term for $M<16$, a comparable term at 16, and a larger one above 16. A ring AllGather also uses half as many rounds as a ring AllReduce. Different backend schedules can change the latency comparison.

\begin{table}[htbp]
\centering
\caption{Communication costs for one sign buffer. Here $s_1=\lceil d/8\rceil$ and $s_8=d$. Ignoring byte padding, AllGather has a smaller bandwidth term for $1<M<16$, but stores all workers' buffers locally.}
\label{tab:comm-allgather-vs-allreduce}
\small
\begin{tabularx}{\linewidth}{@{}lXX@{}}
\toprule
Quantity & Packed AllGather & Int8 AllReduce \\
\midrule
Local payload & $s_1$ & $s_8$ \\
Bytes sent & $(M-1)s_1$ & $2(1-1/M)s_8$ \\
Bytes received & $(M-1)s_1$ & $2(1-1/M)s_8$ \\
Ring rounds & $M-1$ & $2(M-1)$ \\
Ring bandwidth term & $(M-1)s_1\beta$ & $2(1-1/M)s_8\beta$ \\
Result buffer & $Ms_1$ bytes before local voting & $s_8$ bytes of sums \\
\bottomrule
\end{tabularx}
\end{table}

\paragraph{Comparison with float32.}
A float32 buffer contains $4d$ bytes. Int8 uses one quarter as many bytes, and a packed sign representation uses approximately one thirty-second as many. Only the int8-versus-float32 AllReduce comparison holds the collective fixed. For packed AllGather versus float32 AllReduce, the bandwidth-term ratio is approximately $M/64$. A 32-fold representation reduction therefore does not imply a 32-fold collective speedup.

\paragraph{Comparison with signSGD.}
SignMuon and signSGD have the same aggregation cost when they use the same encoding and collective. SignMuon adds local polar computation. Its benefit must therefore come from the resulting optimization behavior or from a particular balance between local computation and communication, rather than a smaller sign payload than signSGD.

\subsection{Parameter-server variants}

The preceding collectives distribute communication among workers. Here we consider two alternative aggregation arrangements: a dedicated server with serialized incoming messages, and a tree that combines partial sums along its edges. Theorems~\ref{thm:ps-star} and~\ref{thm:ps-tree} quantify how these arrangements change the number of communication rounds and the load on the aggregation endpoint while preserving the majority-vote computation.

\begin{theorem}[Serialized parameter-server uplink]
\label{thm:ps-star}
Suppose $M$ workers send $s$-byte sign buffers to one server, which computes the vote and returns an $s$-byte result. The server is separate from the $M$ workers. If the uplink serializes the messages and the downlink uses a pipelined tree broadcast over $M+1$ participants, the model gives
\[
T_{\mathrm{PS-star}}\approx M(\alpha+\beta s)
+\alpha\lceil\log_2(M+1)\rceil+\beta s.
\]
The server receives $Ms$ bytes on the uplink.
\end{theorem}
\begin{proof}
The serialized uplink contains $M$ messages, each costing $\alpha+\beta s$. Add the tree-broadcast cost for the result. If the server also sends the result separately to every worker, its downlink sends $Ms$ bytes; the tree model instead distributes forwarding among workers.
\end{proof}

\begin{theorem}[Tree reduction and broadcast]
\label{thm:ps-tree}
For $M\ge2$ participants including the root, the pipelined tree model assigns a reduction of an $s$-byte integer buffer followed by a broadcast the cost
\[
T_{\mathrm{PS-tree}}\approx2\alpha\lceil\log_2M\rceil+2\beta s.
\]
Partial sums must fit the reduction datatype.
\end{theorem}
\begin{proof}
Add the modeled reduction and broadcast costs. Intermediate nodes sum the integer entries; the root thresholds the final sums before broadcasting the vote.
\end{proof}

A server can receive packed sign buffers and unpack them before voting, but an ordinary integer reduction tree needs enough bits to store partial sums. The same distinction between sign encoding and sum representation applies as in AllReduce. These server models also show why a direct serialized uplink becomes a bottleneck as worker count grows.

\subsection{Implementation implications}

Concatenating sign tensors reduces collective startup costs. AllReduce distributes aggregation across workers, while a serialized server collects all incoming traffic at one endpoint. Packed AllGather reduces each worker's input buffer but adds replicated storage and local vote computation. Local polar computation never adds sign-aggregation messages, although it can dominate an iteration on a fast interconnect. Choosing a collective therefore requires measuring local compute and communication together.
\section{SignMuon: convergence guarantees and solver error}
\label{sec:main-convergence}

All optimization results in this appendix concern SignMuon: workers sign their local momentum polar approximations and update with the resulting majority vote. We do not substitute gradient signs for this update. The convergence guarantees require the stated vote-alignment assumptions.

\subsection{Setup and descent inequality}
Let $W\in\R^{m\times n}$, $d=mn$, $r_0=\min(m,n)$, and $g_t=\nabla f(W_t)$. The inner product is $\langle A,B\rangle=\tr(A^\top B)$. We write $\|\cdot\|_1$ for the entrywise norm and $\|\cdot\|_\ast$ for the nuclear norm, dual to $\|\cdot\|_{\op}$. For the compact SVD $X=L\Sigma R^\top$, define $\polar(X)=LR^\top$ on its positive singular subspace and $\polar(0)=0$. This specifies the choice at deficient rank.

\begin{assumption}[Objective and smoothness]
\label{ass:lower_bounded}\label{ass:spectral_smooth}
The function $f$ is continuously differentiable, $f(W)\ge f^\ast> -\infty$, and
\begin{equation}\label{eq:spectral_smooth}
\|\nabla f(W)-\nabla f(W')\|_\ast\le L_\ast\|W-W'\|_{\op}
\end{equation}
for a finite $L_\ast>0$. The initial $W_0$ and finite positive step sizes $\eta_t$ are deterministic. Write $\Delta=f(W_0)-f^\ast$.
\end{assumption}

Take $\lambda=0$ unless $f$ includes the regularization term. For $0\le\beta<1$, worker $p$ uses
\begin{equation}\label{eq:theory_pipeline}
\begin{split}
M_{t+1}^{(p)}&=\beta M_t^{(p)}+(1-\beta)G_t^{(p)},\qquad M_0^{(p)}=0,\\
U_t^{(p)}&=\mathrm{PolarNS}(M_{t+1}^{(p)}),\qquad
S_t^{(p)}=\sign(U_t^{(p)}),\\
\bar S_t&=\sign\Bigl(\sum_p S_t^{(p)}\Bigr),\qquad
D_t=\bar S_t/\sqrt d,\qquad W_{t+1}=W_t-\eta_tD_t.
\end{split}
\end{equation}
The binary convention sends zero to $+1$, including vote ties. The bounds also allow a ternary zero convention if used consistently. Exact polar factors are the idealized version of the same pipeline. Appendix~\ref{app:blocks} accounts for the implementation's unnormalized signs and multiple blocks. We assume the local routines return finite real matrices and integer aggregation avoids overflow. We take every horizon $T\ge1$ to be an integer, and every schedule constant $\eta_0$ to be positive. No statement here covers the optional polar transformation \emph{after} the vote. The smoothness assumption is a property of the full objective; we do not infer it from a neural-network architecture. In particular, objectives with nonsmooth activations require separate analysis.

\begin{lemma}[Descent inequality]\label{lem:spectral_descent}
Under Assumption~\ref{ass:spectral_smooth},
\[
f(W+H)\le f(W)+\langle\nabla f(W),H\rangle+\tfrac{L_\ast}{2}\|H\|_{\op}^2.
\]
In particular, SignMuon satisfies
$f(W_{t+1})\le f(W_t)-\eta_t\langle g_t,D_t\rangle+L_\ast\eta_t^2/2$.
\end{lemma}
\begin{proof}
Integrate $\langle\nabla f(W+sH)-\nabla f(W),H\rangle$ over $s\in[0,1]$ and use nuclear--spectral duality and \eqref{eq:spectral_smooth}. Finally, $\|D_t\|_{\op}\le\|D_t\|_{\F}\le1$.
\end{proof}

\subsection{Alignment of the actual vote}
Condition on $W_t$ and define
\begin{equation}\label{eq:sign_error_def}
q_{ij,t}=\Pr(\bar S_{t,ij}\ne\sign(g_{t,ij})\mid W_t),\quad
V_t=\frac{\|g_t\|_1}{\sqrt d},\quad
e_t=\frac1{\sqrt d}\sum_{ij}|g_{t,ij}|q_{ij,t}.
\end{equation}
Coordinates with $g_{t,ij}=0$ contribute zero under either sign convention. The probability includes momentum history, gradient sampling, the polar routine, and the vote. Unconditional expectations average over all algorithmic randomness.

\begin{lemma}[Coordinate alignment]\label{lem:inner_product_identity}
For the SignMuon vote, $\E[\langle g_t,D_t\rangle\mid W_t]\ge V_t-2e_t$. Equality holds for binary signs.
\end{lemma}
\begin{proof}
For $g_{t,ij}\ne0$, the binary product $g_{t,ij}\bar S_{t,ij}$ equals $|g_{t,ij}|$ on agreement and $-|g_{t,ij}|$ otherwise. A zero sign replaces the latter value by zero and only increases the expectation. Sum over entries and divide by $\sqrt d$.
\end{proof}

\begin{theorem}[Weighted SignMuon bound]\label{thm:generic_signmuon}
Under Assumption~\ref{ass:spectral_smooth}, define $A_T=\sum_{t<T}\eta_t$, $B_T=\sum_{t<T}\eta_t^2$, and
$\mathcal G_T^\eta=A_T^{-1}\sum_{t<T}\eta_t\E V_t$. Then
\begin{equation}\label{eq:weighted_signmuon}
\mathcal G_T^\eta\le\frac{\Delta}{A_T}+\frac{L_\ast B_T}{2A_T}
+\frac2{A_T}\sum_{t<T}\eta_t\E e_t.
\end{equation}
If $2\E e_t\le\rho\E V_t+\delta_t$ for fixed $0\le\rho<1$ and deterministic $\delta_t\ge0$, then
\begin{equation}\label{eq:conditional_signmuon}
\mathcal G_T^\eta\le\frac1{1-\rho}\left[
\frac{\Delta}{A_T}+\frac{L_\ast B_T}{2A_T}+\frac1{A_T}\sum_{t<T}\eta_t\delta_t\right].
\end{equation}
\end{theorem}
\begin{proof}
Take expectations in Lemma~\ref{lem:spectral_descent}, apply Lemma~\ref{lem:inner_product_identity}, and sum. Lower boundedness gives
$\sum_{t<T}\eta_t\E V_t\le\Delta+L_\ast B_T/2+2\sum_{t<T}\eta_t\E e_t$.
Insert the relative-error assumption and collect the $V_t$ terms for the second bound. Bounded directions place every finite-horizon trajectory in a deterministic compact set. Continuity of $f$ and $\nabla f$ therefore makes the expectations in this proof finite.
\end{proof}

If $A_T\to\infty$, $B_T/A_T\to0$, and $\sum_{t<T}\eta_t\delta_t/A_T\to0$, the conditional bound gives $\mathcal G_T^\eta\to0$. For example, $\eta_t=\eta_0/\sqrt{t+1}$ and $\delta_t=0$ give $O(\log(T+1)/\sqrt T)$. For a fixed horizon, $\eta_t=\eta_0/\sqrt T$ gives $O(T^{-1/2})$ when the average $\delta_t$ has that order. For an output index $R$ independent of the run with $\Pr(R=t)=\eta_t/A_T$, the bound is exactly on $\E[\|\nabla f(W_R)\|_1/\sqrt d]$. A positive residual gives only a bound on this average stationarity measure, not an eventual neighborhood containing every iterate. Neither a residual upper bound nor a stationarity bound establishes a positive lower bound on training error, zero training error, or convergence of the last iterate.

\subsection{The local polar solver and the loss from signing}
\label{app:inexact-signmuon}

An operator-norm approximation can become inaccurate along nearly null singular directions without losing much inner product with momentum. A linear minimization oracle (LMO) solves $\min_{\|Z\|_{\op}\le1}\langle X,Z\rangle$, whose solution includes $-\polar(X)$. Following the objective-value criterion of \citet[Assumptions 2.1--2.2]{qian2026degenerate}, we distinguish these two errors. We derive the following bounds for SignMuon's cubic routine.

\begin{theorem}[Cubic Newton--Schulz on the positive singular subspace]
\label{thm:newton_schulz}
Let $X\ne0$ have rank $r$, let $\alpha\ge\|X\|_{\op}$, and set
\[
Y_0=X/\alpha,\qquad Y_{k+1}=\tfrac12Y_k(3I-Y_k^\top Y_k),\qquad Q=\polar(X).
\]
Let $\sigma_{\min}^+(X)$ be the least positive singular value and $K\ge0$ an integer. For $\rho=1-(\sigma_{\min}^+(X)/\alpha)^2\in[0,1)$,
\begin{equation}\label{eq:ns_operator}
\|Y_K\|_{\op}\le1,\qquad \|Y_K-Q\|_{\op}\le\rho^{2^K}.
\end{equation}
The LMO objective gap also satisfies the bound
\begin{equation}\label{eq:ns_gap}
0\le\epsilon_K(X):=\|X\|_\ast-\langle X,Y_K\rangle
\le r\alpha(2/3)^K.
\end{equation}
The latter requires no positive lower bound on $\sigma_{\min}^+(X)/\alpha$. If $X=0$, choose any $\alpha>0$; all iterates, the compact polar factor, and the gap are zero.
\end{theorem}
\begin{proof}
Write the compact SVD $X=L\Sigma R^\top$. Induction gives $Y_k=L\diag(y_{i,k})R^\top$ with
\[
y_{i,0}=\sigma_i/\alpha\in(0,1],\qquad
y_{i,k+1}=y_{i,k}(3-y_{i,k}^2)/2.
\]
Null singular directions remain zero, including for wide or rank-deficient matrices. For $e_{i,k}=1-y_{i,k}^2$,
\[
e_{i,k+1}=\tfrac34e_{i,k}^2+\tfrac14e_{i,k}^3\le e_{i,k}^2.
\]
Thus $0<y_{i,k}\le1$, $e_{i,K}\le\rho^{2^K}$, and $1-y_{i,K}\le e_{i,K}$. This proves \eqref{eq:ns_operator} without assuming that $I-Y_k^\top Y_k$ contracts on the nullspace.

For the second bound, set $z_{i,k}=y_{i,k}^{-2}-1=e_{i,k}/(1-e_{i,k})$. If $e_{i,k}>0$, direct algebra gives
\[
\frac{z_{i,k+1}}{z_{i,k}}=
\frac{e_{i,k}(3+e_{i,k})}{(2+e_{i,k})^2}\le\frac49.
\]
The inequality follows from $(e-1)(5e+16)\le0$ for $e\in[0,1]$; $e=0$ remains zero. Since $1-y\le\sqrt{y^{-2}-1}$ on $(0,1]$, with $x_i=\sigma_i/\alpha$ we obtain
\[
\sigma_i(1-y_{i,K})\le
\alpha x_i(2/3)^K\sqrt{x_i^{-2}-1}
=\alpha(2/3)^K\sqrt{1-x_i^2}\le\alpha(2/3)^K.
\]
Summing these nonnegative terms proves \eqref{eq:ns_gap}.
\end{proof}

With $\alpha=\|X\|_{\F}$, \eqref{eq:ns_gap} yields
$\epsilon_K(X)\le r(2/3)^K\|X\|_\ast$. Thus, for fixed rank and sufficiently large $K$, the relative inner-product error can be uniformly small even when the least positive singular value approaches zero. A scaling floor $\alpha=\max\{\|X\|_{\F},\varepsilon_{\mathrm{scale}}\}$ instead gives
\[
\epsilon_K(X)\le r(2/3)^K(\|X\|_\ast+\varepsilon_{\mathrm{scale}}).
\]
The implementation's additive scaling $\alpha=\|X\|_{\F}+\varepsilon_{\mathrm{scale}}$ obeys the same bound. Both choices introduce an additive tolerance. These are exact-arithmetic statements for the cubic polynomial and a certified upper scaling bound. They do not certify finite power-iteration norm estimates, other Newton--Schulz polynomials, or floating-point rounding.

\begin{proposition}[A SignMuon loss decomposition]\label{prop:signmuon_gap}
For one worker, let $M$ be its current momentum, $U$ satisfy $\|U\|_{\op}\le1$, and $D=\sign(U)/\sqrt d$. In particular, $U=Y_K(M)$ qualifies under the scaling assumptions of Theorem~\ref{thm:newton_schulz}. Define
\[
\epsilon_{\mathrm{NS}}=\|M\|_\ast-\langle M,U\rangle,\qquad
\chi_{\mathrm{sign}}=\bigl[\langle M,U-D\rangle\bigr]_+,
\quad [a]_+=\max\{a,0\}.
\]
Then $\epsilon_{\mathrm{NS}}\ge0$ and, for any gradient $g$,
\begin{equation}\label{eq:signmuon_gap}
\langle g,D\rangle\ge\|g\|_\ast
-2\|g-M\|_\ast-\epsilon_{\mathrm{NS}}-\chi_{\mathrm{sign}}.
\end{equation}
\end{proposition}
\begin{proof}
Feasibility $\|U\|_{\op}\le1$ and duality imply nonnegativity of the oracle gap. Also,
\[
\langle M,D\rangle=\|M\|_\ast-\epsilon_{\mathrm{NS}}-\langle M,U-D\rangle
\ge\|M\|_\ast-\epsilon_{\mathrm{NS}}-\chi_{\mathrm{sign}}.
\]
Use $\langle g-M,D\rangle\ge-\|g-M\|_\ast$ and
$\|M\|_\ast\ge\|g\|_\ast-\|g-M\|_\ast$.
\end{proof}

The feasibility assumption is substantive. For example, $X=\diag(1,4)$ scaled by an underestimate $\alpha=1$ gives $Y_1=\diag(1,-26)$, which lies outside the spectral unit ball. A finite power-iteration estimate does not certify Proposition~\ref{prop:signmuon_gap}. The stationarity bounds based directly on the final signs remain valid under their own alignment conditions, because their direction bound does not require $U$ to be feasible.

The signing loss includes the change of scale from the spectral-ball direction $U$ to $D=\sign(U)/\sqrt d$ as well as the change of signs. The solver bound controls $\epsilon_{\mathrm{NS}}$, not $\chi_{\mathrm{sign}}$. An exact polar factor has zero oracle gap; the separate signing term still enters the alignment condition. If
\begin{equation}\label{eq:sufficient_momentum_alignment}
\epsilon_{\mathrm{NS}}+\chi_{\mathrm{sign}}\le(1-\kappa)\|M\|_\ast+\varepsilon_t,
\qquad 0<\kappa\le1,
\end{equation}
then $\langle M,D\rangle\ge\kappa\|M\|_\ast-\varepsilon_t$. The momentum theorem uses this condition to bound the gradient norm.

For distributed SignMuon with every local $U_t^{(p)}$ feasible, apply \eqref{eq:signmuon_gap} separately to the local normalized signs $D_t^{(p)}$. Writing $D_t$ for their normalized vote, define
\[
\chi_{\mathrm{vote},t}=\left[\left\langle g_t,\frac1M\sum_pD_t^{(p)}-D_t\right\rangle\right]_+.
\]
Averaging the local inequalities gives the valid aggregate bound
\begin{equation}\label{eq:distributed_gap}
\langle g_t,D_t\rangle\ge\|g_t\|_\ast
-\frac1M\sum_p\left(2\|g_t-M_{t+1}^{(p)}\|_\ast
+\epsilon_{\mathrm{NS},t}^{(p)}+\chi_{\mathrm{sign},t}^{(p)}\right)
-\chi_{\mathrm{vote},t}.
\end{equation}
Thus even local alignment does not by itself establish an aligned majority vote. The aggregate bound therefore includes the vote-loss term.

\subsection{Conditional fixed-batch convergence for SignMuon}
\label{app:momentum-rate}

We bound momentum-tracking error for the pipeline in \eqref{eq:theory_pipeline}. The proof adapts the inexact-oracle approach of \citet[Theorem 3.6]{qian2026degenerate}, but places its alignment requirement on the \emph{final SignMuon vote}. No condition on the unsigned polar factor can replace this requirement.

Let $(\mathcal F_t)$ be the filtration of histories before sampling each step $t$, including $W_t$ and all momentum buffers. Suppose
\[
\bar G_t=\frac1M\sum_pG_t^{(p)},\qquad
\E[\bar G_t\mid\mathcal F_t]=g_t,\qquad
\E[\|\bar G_t-g_t\|_{\F}^2\mid\mathcal F_t]\le\bar\sigma^2.
\]
For a single worker, this is the usual fixed-mini-batch variance bound. Conditional independence and worker errors centered at $\E[G_t^{(p)}\mid\mathcal F_t]$, with conditional second moments at most $\sigma^2$ give $\bar\sigma\le\sigma/\sqrt M$. If worker means differ, their average must still equal $g_t$. Define the mean momentum
$\bar M_{t+1}=M^{-1}\sum_pM_{t+1}^{(p)}$.

\begin{theorem}[SignMuon under momentum--vote alignment]\label{thm:momentum_signmuon}
Under Assumption~\ref{ass:spectral_smooth} and the conditional sampling conditions above, use a constant $\eta>0$ and $\beta\in[0,1)$ for $T$ steps. Let $\varepsilon_t$ be finite deterministic numbers. Assume for every $t<T$ that the actual SignMuon vote satisfies
\begin{equation}\label{eq:momentum_vote_condition}
\E\langle \bar M_{t+1},D_t\rangle\ge\kappa\E\|\bar M_{t+1}\|_\ast-\varepsilon_t,
\qquad 0<\kappa\le1,\quad\varepsilon_t\ge0.
\end{equation}
Then, with $\bar\varepsilon_T=T^{-1}\sum_{t<T}\varepsilon_t$,
\begin{align}\label{eq:momentum_rate}
\frac1T\sum_{t<T}\E\|g_t\|_\ast
\le\frac1\kappa\bigg[&\frac{\Delta}{\eta T}+\frac{L_\ast\eta}{2}+\bar\varepsilon_T\nonumber\\
&+(1+\kappa)\sqrt{r_0}\left(
\frac{\beta\|g_0\|_{\F}}{T(1-\beta)}
+\frac{L_\ast\eta\beta}{1-\beta}
+\bar\sigma\sqrt{\frac{1-\beta}{1+\beta}}\right)\bigg].
\end{align}
In particular, $\eta=\eta_0T^{-3/4}$, $\beta=1-T^{-1/2}$, and $\bar\varepsilon_T=O(T^{-1/4})$ give $O(T^{-1/4})$ expected average nuclear-gradient norm for fixed dimension, variance, and $\kappa>0$ independent of $T$.
\end{theorem}
\begin{proof}
Let $\zeta_t=\bar G_t-g_t$ and $\bar M_0=0$. Expanding the momentum recursion yields
\begin{equation}\label{eq:tracking_expansion}
\bar M_{t+1}-g_t=-\beta^{t+1}g_0
+\sum_{s=1}^{t}\beta^{t-s+1}(g_{s-1}-g_s)
+(1-\beta)\sum_{s=0}^{t}\beta^{t-s}\zeta_s.
\end{equation}
The update bound and smoothness imply
$\|g_s-g_{s-1}\|_{\F}\le\|g_s-g_{s-1}\|_\ast\le L_\ast\eta$.
The martingale differences $\zeta_s$ have zero cross terms in the expected squared Frobenius norm: for $u<s$, $\zeta_u$ is $\mathcal F_s$-measurable and $\E[\zeta_s\mid\mathcal F_s]=0$. Thus
\[
\E\left\|(1-\beta)\sum_{s=0}^t\beta^{t-s}\zeta_s\right\|_{\F}
\le\bar\sigma\sqrt{\frac{1-\beta}{1+\beta}}.
\]
Triangle inequalities in \eqref{eq:tracking_expansion} give
\[
\E\|\bar M_{t+1}-g_t\|_\ast\le\sqrt{r_0}\left(
\beta^{t+1}\|g_0\|_{\F}+\frac{L_\ast\eta\beta}{1-\beta}
+\bar\sigma\sqrt{\frac{1-\beta}{1+\beta}}\right).
\]
By duality, $\|D_t\|_{\op}\le1$, and \eqref{eq:momentum_vote_condition},
\[
\E\langle g_t,D_t\rangle\ge\kappa\E\|g_t\|_\ast
-(1+\kappa)\E\|g_t-\bar M_{t+1}\|_\ast-\varepsilon_t.
\]
Insert this inequality into Lemma~\ref{lem:spectral_descent}, telescope, and use
$\sum_{t<T}\beta^{t+1}\le\beta/(1-\beta)$. This proves \eqref{eq:momentum_rate}. Substitution of the stated schedules gives the claimed order: the objective term, drift term, and variance term have order $T^{-1/4}$; the initialization term has order at most $T^{-1/2}$.
\end{proof}

This is a horizon-dependent family of runs, not a claim about an infinite run with fixed momentum. The condition concerns $\bar M_{t+1}$ and the \emph{voted} signs, although workers compute polar factors from their own momenta. It therefore does not assert that the implementation polarizes the mean momentum. With one worker, \eqref{eq:sufficient_momentum_alignment} suffices; with several workers, the vote requires a separate check. A $1/\sqrt M$ variance reduction under independent sampling does not imply a corresponding improvement in $\kappa$ or the signing loss.

\paragraph{Extension to biased inputs.}
\label{app:biased-momentum}
If the gradients entering the buffer are biased, set
\[
b_t=\E[\bar G_t\mid\mathcal F_t]-g_t,\qquad
\zeta_t=\bar G_t-\E[\bar G_t\mid\mathcal F_t],\qquad
\bar b_T=\frac1T\sum_{t<T}\E\|b_t\|_\ast<\infty.
\]
Assume the same conditional variance bound for the centered $\zeta_t$. Then \eqref{eq:momentum_rate} holds with the additional term $(1+\kappa)\bar b_T/\kappa$ on its right-hand side. To prove this, expansion \eqref{eq:tracking_expansion} acquires
$R_t=(1-\beta)\sum_{s=0}^t\beta^{t-s}b_s$.
The remaining noise is still a martingale difference. The geometric weights give
\[
\frac1T\sum_{t<T}\E\|R_t\|_\ast
\le\frac1T\sum_{s<T}\E\|b_s\|_\ast
\sum_{t=s}^{T-1}(1-\beta)\beta^{t-s}\le\bar b_T.
\]
Insert this estimate into the same alignment and descent inequalities. Thus the fixed-batch rate additionally requires $\bar b_T=O(T^{-1/4})$ when the input is biased. A fixed clipping threshold does not imply this decay.

\paragraph{A setting that ensures the additional condition.}
For one worker, suppose $f$ depends only on the diagonal entries of a square parameter matrix and every sampled gradient is diagonal. Starting from zero momentum, the buffer and all cubic solver iterates remain diagonal. Positive scaling and the scalar recurrence preserve the sign of every nonzero diagonal entry, so
\[
\langle M_{t+1},D_t\rangle=\frac{\|M_{t+1}\|_\ast}{\sqrt d}.
\]
Zero diagonal entries contribute nothing; the binary convention on off-diagonal zeros does not affect this inner product or the objective. Thus \eqref{eq:momentum_vote_condition} holds with $\kappa=1/\sqrt d$ and $\varepsilon_t=0$ for any finite cubic budget, including zero iterations, under a valid scaling bound. This restricted model shows that the condition is satisfiable. It does not establish alignment for dense neural-network gradients or for a majority vote of different worker momenta.

\subsection{Smoothness, multiple blocks, and implementation scope}
\label{app:theory-scope}

This subsection connects the single-matrix analysis to the multi-block updates in Algorithms~\ref{alg:sign-muon} and~\ref{alg:sign-muon-dist-1bit}. We first compare smoothness constants and account for block-dependent normalization, then state the joint momentum bound in Theorem~\ref{thm:joint-momentum}. We then account for regularization, clipping, and finite polar accuracy.

\subsubsection{Comparing spectral and Frobenius smoothness}
\label{app:smoothness_comparison}

Let $r_0=\min\{m,n\}$. Denote by $L_{\F}^{\min}$ and $L_\ast^{\min}$ the least admissible global constants in
\[
\|\nabla f(W)-\nabla f(W')\|_{\F}\le L_{\F}\|W-W'\|_{\F}
\]
and Assumption~\ref{ass:spectral_smooth}, respectively. When either is finite, both are finite and
\begin{equation}
\label{eq:smoothness_constants}
L_{\F}^{\min}\le L_\ast^{\min}\le r_0L_{\F}^{\min}.
\end{equation}
To prove the first inequality, combine $\|A\|_{\F}\le\|A\|_\ast$ and $\|B\|_{\op}\le\|B\|_{\F}$ with spectral smoothness. Conversely, Frobenius smoothness gives
\[
\|\nabla f(W)-\nabla f(W')\|_\ast
\le\sqrt{r_0}L_{\F}\|W-W'\|_{\F}
\le r_0L_{\F}\|W-W'\|_{\op}.
\]
Taking the least constants proves \eqref{eq:smoothness_constants}. For $f(W)=\|W\|_{\F}^2/2$, the constants are $L_{\F}^{\min}=1$ and $L_\ast^{\min}=r_0$; a rectangular diagonal matrix with $r_0$ unit singular values attains the latter ratio.

For a normalized sign direction, $\|S/\sqrt{mn}\|_{\F}\le1$ as well as $\|S/\sqrt{mn}\|_{\op}\le1$. Thus every single-matrix sign bound above also holds with a valid $L_{\F}$ in place of $L_\ast$ if one uses the Frobenius descent lemma. Spectral normalization alone does not improve the sign method's curvature term. An unquantized polar direction can instead have squared Frobenius norm as large as $r_0$, so its Frobenius curvature bound involves $r_0L_{\F}$. We retain the nuclear--spectral formulation to describe the local polar operation, without claiming an automatic curvature advantage for the signed update.

\subsubsection{Simultaneous updates of several parameter blocks}
\label{app:blocks}

Let $W=(W^{(1)},\ldots,W^{(J)})$ contain all jointly updated blocks, with $d_\ell=m_\ell n_\ell$. A vector is a one-column block. Define the product norm and its dual by
\[
\|W\|_{\Pi}=\max_\ell\|W^{(\ell)}\|_{\op},\qquad
\|G\|_{\Pi,\ast}=\sum_{\ell=1}^J\|G^{(\ell)}\|_\ast.
\]
Their inner product sums the blockwise Frobenius inner products. Assume $f$ is lower bounded and continuously differentiable, and impose the joint condition
\begin{equation}
\label{eq:joint_smoothness}
\|\nabla f(W)-\nabla f(W')\|_{\Pi,\ast}
\le L_\Pi\|W-W'\|_\Pi.
\end{equation}
This condition controls interactions between blocks; smoothness of each block with the others held fixed does not suffice to ignore those interactions.

Choose fixed $c_\ell>0$, let $C=\max_\ell c_\ell$, and update every block as
\[
W_{t+1}^{(\ell)}=W_t^{(\ell)}-\eta_tD_t^{(\ell)},
\qquad D_t^{(\ell)}=c_\ell S_t^{(\ell)}/\sqrt{d_\ell},
\quad S_t^{(\ell)}\in\{-1,0,+1\}^{m_\ell\times n_\ell}.
\]
Let $g_t^{(\ell)}=\nabla_\ell f(W_t)$ and define $q_{ij,t}^{(\ell)}=\Pr(S_{t,ij}^{(\ell)}\ne\sign(g_{t,ij}^{(\ell)})\mid W_t)$, conditioning on the full tuple $W_t$. Define
\[
V_t=\sum_\ell\frac{c_\ell\|g_t^{(\ell)}\|_1}{\sqrt{d_\ell}},
\qquad
E_t=\sum_\ell\frac{c_\ell}{\sqrt{d_\ell}}\sum_{ij}|g_{t,ij}^{(\ell)}|q_{ij,t}^{(\ell)}.
\]
For deterministic $W_0$ and $\eta_t>0$, the joint generic bound is
\begin{equation}
\label{eq:joint_sign_bound}
\frac1{A_T}\sum_{t<T}\eta_t\E V_t
\le\frac{f(W_0)-f^\ast}{A_T}
+\frac{L_\Pi C^2B_T}{2A_T}
+\frac2{A_T}\sum_{t<T}\eta_t\E E_t.
\end{equation}
Indeed, product-norm duality and \eqref{eq:joint_smoothness} give the same integral descent inequality as Lemma~\ref{lem:spectral_descent}, with $\|D_t\|_\Pi\le C$. Summing the coordinate alignment inequalities over blocks gives
$\E[\langle\nabla f(W_t),D_t\rangle\mid W_t]\ge V_t-2E_t$.
Taking expectations and telescoping proves \eqref{eq:joint_sign_bound}. Bounded directions again place every finite-horizon trajectory in a deterministic compact set, which justifies the expectations.

Choosing $c_\ell=1$ gives per-block normalized updates. Choosing $c_\ell=\sqrt{d_\ell}$ gives the implementation's unnormalized signs with one common learning rate; then $C^2=\max_\ell d_\ell$ and $V_t$ is the sum of the blocks' entrywise gradient norms. Thus a common learning-rate rescaling cannot remove all dimension factors for blocks of unequal shape. Here each $S_t^{(\ell)}$ is the actual SignMuon vote for that block. Its error term remains conditional. This extension does not cover additional blocks updated by a different optimizer unless we include their progress terms separately.

\subsubsection{Momentum rate for the full unnormalized method}
\label{app:joint-momentum}

The preceding joint bound concerns sign errors. We now extend the momentum theorem to all blocks, including biased inputs, so that the learning rate and update scales match the displayed algorithms. Use the actual local momentum--polar--sign vote in each block and the update scales $c_\ell$ defined above. Let $\bar G_t^{(\ell)}$ and $\bar M_{t+1}^{(\ell)}$ average the gradients entering the buffers and the updated momenta over workers. With respect to the common pre-sampling history $\mathcal F_t$, define
\[
b_t^{(\ell)}=\E[\bar G_t^{(\ell)}\mid\mathcal F_t]-g_t^{(\ell)},\qquad
\zeta_t^{(\ell)}=\bar G_t^{(\ell)}-\E[\bar G_t^{(\ell)}\mid\mathcal F_t].
\]
Assume finite conditional means and
$\E[\|\zeta_t^{(\ell)}\|_{\F}^2\mid\mathcal F_t]\le\bar\sigma_\ell^2$.
No independence between blocks is necessary. Put $r_\ell=\min(m_\ell,n_\ell)$ and
\begin{equation}\label{eq:joint_constants}
\begin{split}
H_0&=\sum_\ell c_\ell\sqrt{r_\ell}\,\|g_0^{(\ell)}\|_{\F},\qquad
R_c=\sum_\ell c_\ell\sqrt{r_\ell},\\
\Sigma_c&=\sum_\ell c_\ell\sqrt{r_\ell}\,\bar\sigma_\ell,\qquad
\bar b_T=\frac1T\sum_{t<T}\sum_\ell c_\ell\E\|b_t^{(\ell)}\|_\ast<\infty.
\end{split}
\end{equation}

\begin{theorem}[Joint SignMuon momentum bound]\label{thm:joint-momentum}
Assume a continuously differentiable lower-bounded $f$ satisfying \eqref{eq:joint_smoothness}, deterministic $W_0$, zero initial momentum, and the sampling conditions above. Use fixed $\eta>0$, common $\beta\in[0,1)$, and fixed $c_\ell>0$ for $T$ steps. Suppose the actual scaled votes satisfy
\begin{equation}\label{eq:joint_alignment}
\sum_\ell\E\langle\bar M_{t+1}^{(\ell)},D_t^{(\ell)}\rangle
\ge\kappa\sum_\ell c_\ell\E\|\bar M_{t+1}^{(\ell)}\|_\ast-\varepsilon_t
\end{equation}
for fixed $0<\kappa\le1$ and finite deterministic $\varepsilon_t\ge0$. Let
$\Delta=f(W_0)-f^\ast$ and $\bar\varepsilon_T=T^{-1}\sum_{t<T}\varepsilon_t$. Then
\begin{equation}\label{eq:joint_momentum_rate}
\begin{split}
\frac1T\sum_{t<T}\sum_\ell c_\ell\E\|g_t^{(\ell)}\|_\ast
\le\frac1\kappa\bigg[&\frac{\Delta}{\eta T}+\frac{L_\Pi C^2\eta}{2}+\bar\varepsilon_T\\
+(1+\kappa)\bigg(&\frac{\beta H_0}{T(1-\beta)}
+\frac{L_\Pi C\eta\beta R_c}{1-\beta}
+\Sigma_c\sqrt{\frac{1-\beta}{1+\beta}}+\bar b_T\bigg)\bigg].
\end{split}
\end{equation}
Choosing $c_\ell=\sqrt{d_\ell}$ recovers the unnormalized sign updates with their common implementation learning rate $\eta$. With fixed problem constants, $\eta=\eta_0T^{-3/4}$, $\beta=1-T^{-1/2}$, and $\bar\varepsilon_T+\bar b_T=O(T^{-1/4})$, the left-hand side is $O(T^{-1/4})$.
\end{theorem}
\begin{proof}
Since $\|D_t\|_\Pi\le C$, joint smoothness gives
\[
f(W_{t+1})\le f(W_t)-\eta\sum_\ell\langle g_t^{(\ell)},D_t^{(\ell)}\rangle
+\tfrac12L_\Pi C^2\eta^2.
\]
It also controls each block's gradient drift, including cross-block effects:
\[
\|g_s^{(\ell)}-g_{s-1}^{(\ell)}\|_{\F}
\le\|\nabla f(W_s)-\nabla f(W_{s-1})\|_{\Pi,\ast}
\le L_\Pi C\eta.
\]
Apply expansion \eqref{eq:tracking_expansion} to each block, with the additional bias convolution from Appendix~\ref{app:biased-momentum}. The centered noise has zero cross-time terms within each block; the geometric bias weights sum to at most one after averaging over time. Therefore
\[
\frac1T\sum_{t<T}\sum_\ell c_\ell\E\|g_t^{(\ell)}-\bar M_{t+1}^{(\ell)}\|_\ast
\le\frac{\beta H_0}{T(1-\beta)}+\frac{L_\Pi C\eta\beta R_c}{1-\beta}
+\Sigma_c\sqrt{\frac{1-\beta}{1+\beta}}+\bar b_T.
\]
For each block, $\|D_t^{(\ell)}\|_{\op}\le c_\ell$. The triangle inequality and \eqref{eq:joint_alignment} imply
\[
\sum_\ell\E\langle g_t^{(\ell)},D_t^{(\ell)}\rangle
\ge\kappa\sum_\ell c_\ell\E\|g_t^{(\ell)}\|_\ast
-(1+\kappa)\sum_\ell c_\ell\E\|g_t^{(\ell)}-\bar M_{t+1}^{(\ell)}\|_\ast-\varepsilon_t.
\]
Substitute this inequality into joint descent and telescope. This yields \eqref{eq:joint_momentum_rate}; the schedules give the stated order. Bounded directions ensure bounded gradients at each finite horizon, and the assumed noise and bias moments make all momentum expectations used here finite.
\end{proof}

This theorem covers unequal block shapes without treating layers as independent objectives. It still requires the final votes to satisfy \eqref{eq:joint_alignment}. It does not infer this condition from local polar feasibility or worker independence. Blocks updated by a different optimizer require additional progress terms.
 
\subsubsection{Coupled regularization and gradient clipping}
\label{app:regularization-scope}

The implementation adds $\lambda W_t$ to the gradient before momentum. With $\lambda\ge0$, this is the gradient contribution of the coupled objective
\[
f_\lambda(W)=f(W)+\tfrac\lambda2\|W\|_{\F}^2,
\qquad \nabla f_\lambda(W)=\nabla f(W)+\lambda W.
\]
If $f$ satisfies Assumption~\ref{ass:spectral_smooth}, then $f_\lambda$ has the valid constant $L_{\ast,\lambda}=L_\ast+\lambda r_0$, because
\[
\|\nabla f_\lambda(W)-\nabla f_\lambda(W')\|_\ast
\le (L_\ast+\lambda r_0)\|W-W'\|_{\op}.
\]
For several blocks with coefficients $\lambda_\ell\ge0$, the corresponding joint constant is $L_{\Pi,\lambda}=L_\Pi+\sum_\ell\lambda_\ell r_\ell$: sum the blockwise inequalities $\lambda_\ell\|W^{(\ell)}-W'^{(\ell)}\|_\ast\le\lambda_\ell r_\ell\|W-W'\|_\Pi$.
The regularized objective remains lower bounded by $f^\ast$. Adding the deterministic term $\lambda W_t$ to an unbiased estimator preserves its conditional variance and centered noise distribution, but changes its mean to $\nabla f_\lambda(W_t)$. A SignMuon alignment condition using this target concerns stationarity of $f_\lambda$, not of the unregularized $f$. This operation differs from decoupled parameter shrinkage. Subsequent momentum and polar transformation still require their own alignment analysis.

Global norm clipping with threshold $\tau>0$ maps $Z$ to
$C_\tau(Z)=Z\min\{1,\tau/\|Z\|_{\F}\}$, with $C_\tau(0)=0$. It preserves the signs of a sampled gradient but can change its conditional mean and its later momentum direction. If $\E[Z_t\mid\mathcal F_t]=g_t$ for one matrix, its clipping bias satisfies
\begin{equation}\label{eq:clipping_bias}
\left\|\E[C_\tau(Z_t)\mid\mathcal F_t]-g_t\right\|_\ast
\le\sqrt{r_0}\,\E[(\|Z_t\|_{\F}-\tau)_+\mid\mathcal F_t].
\end{equation}
Indeed, conditional Jensen and $\|A\|_\ast\le\sqrt{r_0}\|A\|_{\F}$ reduce this to
$\|C_\tau(Z)-Z\|_{\F}=(\|Z\|_{\F}-\tau)_+$. For global clipping of the concatenated gradient, each block's Frobenius change is at most this same total excess norm, so the corresponding blockwise bound also follows. Deterministic coupled regularization added after clipping does not remove the clipping bias relative to $\nabla f_\lambda$.

Even a scalar unbiased sample taking values $2$ and $-1$ with equal probability has mean $1/2$ but clipped mean zero at $\tau=1$. A fixed threshold therefore does not supply the vanishing-bias premise of Appendix~\ref{app:biased-momentum} or Theorem~\ref{thm:joint-momentum}. The sign-error bound can still describe clipped SignMuon directly, subject to its actual vote-alignment condition. The experimental fixed momentum, constant learning rates, and clipping are not the horizon-dependent settings of the fixed-batch rate theorem.

\subsubsection{When polar approximation preserves signs}
\label{app:sign_stability}

Let $Q=\polar(X)$ and let $U$ satisfy $\|U-Q\|_{\op}\le\varepsilon$. Use the same sign convention for $U$ and $Q$, including at zero. Since
\[
|U_{ij}-Q_{ij}|=|e_i^\top(U-Q)e_j|\le\varepsilon,
\]
the two signs agree whenever $|Q_{ij}|>\varepsilon$. For any gradient matrix $g$, this yields
\begin{equation}
\label{eq:polar_sign_approx_alignment}
\left|\left\langle g,\frac{\sign(U)-\sign(Q)}{\sqrt d}\right\rangle\right|
\le\frac2{\sqrt d}\sum_{ij}|g_{ij}|\,\mathbf1\{|Q_{ij}|\le\varepsilon\}.
\end{equation}
To prove the inequality, discard coordinates whose signs agree and use the bound two on the difference of any two signs. The remaining coordinates all lie in the indicated small-entry set. The same estimate holds for the ternary convention.

The bound combines operator-norm accuracy with entrywise margins: all signs agree if $|Q_{ij}|>\varepsilon$ for every entry. Near-zero entries may change sign. For distributed updates, the vote-loss term in \eqref{eq:distributed_gap} separately accounts for aggregation.

\subsection{Relationship to recent theory}
\label{app:related-theory}

\citet{qian2026degenerate} analyze inexact linear minimization through objective-value error. We use that viewpoint for the cubic solver in Theorem~\ref{thm:newton_schulz}, separate signing loss in Proposition~\ref{prop:signmuon_gap}, and require final-vote alignment in Theorem~\ref{thm:momentum_signmuon}. Their results motivate this analysis but do not establish our alignment premise. Our proofs use ordinary joint smoothness and coupled regularization; their generalized-smoothness and decoupled-weight-decay conclusions do not transfer automatically.

\end{document}